\documentclass[11pt]{article}

\usepackage{fullpage}
\usepackage[round]{natbib}

\usepackage{amsmath,amsthm,amsfonts,amssymb}
\usepackage{amsmath}
\usepackage{hyperref}
\usepackage{color}
\usepackage{mathrsfs}
\usepackage{bm}
\usepackage{multirow}
\usepackage{booktabs}
\usepackage{makecell}
\usepackage{graphicx}
\usepackage{caption}
\usepackage{subcaption}
\usepackage{thmtools}
\usepackage{thm-restate}
\usepackage{setspace}
\usepackage{makecell}
\definecolor{light-gray}{gray}{0.85}
\usepackage{colortbl}
\usepackage{xcolor}
\usepackage{hhline}

\usepackage[round]{natbib}
\usepackage{hyperref}       
\usepackage{url}    
\usepackage{booktabs}       
\usepackage{nicefrac}       
\usepackage{microtype}  

\usepackage{amsmath,amsthm,amsxtra,graphicx,verbatim,epsfig,color,enumerate,array,mathtools,dsfont,multicol,mathrsfs}

\usepackage[ruled,linesnumbered,vlined]{algorithm2e}
\usepackage{algpseudocode}
\SetKwInOut{Parameter}{parameter}

\hypersetup{
    colorlinks,
    linkcolor={blue!50!black},
    citecolor={blue!50!black},
}
\colorlet{linkequation}{blue}

  {\list{}{\leftmargin=0.3in\rightmargin=0.3in}\item[]}%
  {\endlist}
  
\usepackage[ruled,linesnumbered,vlined]{algorithm2e}
\usepackage{algpseudocode}
\SetKwInOut{Parameter}{parameter}

\usepackage{amsfonts}

\newtheorem{theorem}{Theorem}[section]
\newtheorem{lemma}[theorem]{Lemma}
\newtheorem{corollary}[theorem]{Corollary}
\newtheorem{remark}[theorem]{Remark}

\newtheorem{claim}{Claim}

\theoremstyle{definition}
\newtheorem{definition}[theorem]{Definition}

\newtheorem{assumption}{Assumption}

\newcommand{\lep}[1]{\mathop  \le \limits^{(#1)}}
\newcommand{\gep}[1]{\mathop  \ge \limits^{(#1)}}
\newcommand{\ep}[1]{\mathop  = \limits^{(#1)}}
\newcommand{\ex}[1]{\mathbb{E}\left[ #1 \right] }

\newcommand{\cN}{\mathcal{N}}

\newcommand{\Real}{\mathbb{R}}

\newcommand{\norm}[1]{\left\lVert#1\right\rVert}

\newcommand{\inner}[2]{\langle #1, #2 \rangle}

\newcommand{\argmax}{\mathop{\mathrm{argmax}}}

\newcommand{\argmin}{\mathop{\mathrm{argmin}}}

\newcommand{\beq}{\begin{equation}}
\newcommand{\eeq}{\end{equation}}
\newcommand{\beqn}{\begin{equation*}}
\newcommand{\eeqn}{\end{equation*}}
\newcommand{\beqa}{\begin{eqnarray}}
\newcommand{\eeqa}{\end{eqnarray}}
\newcommand{\beqan}{\begin{eqnarray*}}
\newcommand{\eeqan}{\end{eqnarray*}}

\renewcommand{\epsilon}{\varepsilon}

\renewcommand{\hat}{\widehat}

\usepackage{bbm}
\usepackage{times}

\newcommand{\pt}[1]{\mathbb{P}_t\left( #1 \right) }
\newcommand{\ext}[1]{\mathbb{E}_t\left[ #1 \right] }

\begin{document}

\title{On Kernelized Multi-Armed Bandits with Constraints}

\author{%
Xingyu Zhou\thanks{Wayne State University. Email: \texttt{xingyu.zhou@wayne.edu}}\quad Bo Ji\thanks{Virginia Tech. Email: \texttt{boji@vt.edu}} }

\date{}

\maketitle

\begin{abstract}
We study a stochastic bandit problem with a general unknown reward function and a general unknown constraint function. Both functions can be non-linear (even non-convex) and are assumed to lie in a reproducing
kernel Hilbert space (RKHS) with a bounded norm. This kernelized bandit setup strictly generalizes standard multi-armed bandits and linear bandits. 
In contrast to safety-type hard constraints studied in prior works, we consider soft constraints that may be violated in any round as long as the cumulative violations are small, which is motivated by various practical applications.
Our ultimate goal is to study how to utilize the nature of soft constraints to attain a finer complexity-regret-constraint trade-off in the kernelized bandit setting. 
To this end, leveraging primal-dual optimization, we propose a general framework for both algorithm design and performance analysis. This framework builds upon a novel sufficient condition, which not only is satisfied under general exploration strategies, including \emph{upper confidence bound} (UCB), \emph{Thompson sampling} (TS), and new ones based  on \emph{random exploration}, but also enables a unified analysis for showing both sublinear regret and sublinear or even zero constraint violation. We demonstrate the superior performance of our proposed algorithms via numerical experiments based on both synthetic and real-world datasets. 
Along the way, we also make the first detailed comparison between two popular methods for analyzing constrained bandits and Markov decision processes (MDPs) by discussing the key difference and some subtleties in the analysis, which could be of independent interest to the communities.
\end{abstract}

\section{Introduction}

Stochastic bandit optimization of an unknown function $f$ over a domain $\mathcal{X}$ has recently gained increasing popularity due to its widespread real-life applications such as recommendations~\citep{li2010contextual}, cloud resource configurations~\citep{thananjeyan2021pac}, and wireless power control~\citep{chiang2008power}. At each time $t$, an action $x_t$ is selected and then a (noisy) bandit reward feedback $y_t$ is observed. The goal is to maximize the cumulative reward, or equivalently minimize the total regret due to not choosing the optimal action in hindsight. A classic model of this problem is the multi-armed bandits (MAB) where $\mathcal{X}$ consists of finite independent actions. To handle a large action space and correlation among actions, MAB was later generalized to linear bandits where $\mathcal{X}$ is now a subset of $\mathbb{R}^d$ and $f$ is a linear function with respect to a $d$-dimensional feature vector associated with each action. However, in many aforementioned applications, the unknown function $f$ cannot be well-parameterized by a linear function. To this end, researchers have turned to nonparametric models of $f$ via Gaussian process or reproducing kernel Hilbert space (RKHS), which are able to uniformly approximate an arbitrary continuous function over a compact set~\citep{micchelli2006universal}. In this paper, as in~\citep{chowdhury2017kernelized,srinivas2009gaussian}, we consider the agnostic setting (i.e., frequentist-type), where $f$ is assumed to be a fixed function in an RKHS with a bounded norm (i.e., a measure of smoothness). We call this setting \emph{frequentist-type kernelized bandits} (KB).


In addition to a non-linear or even non-convex function $f$ in the above practical applications, another common feature is that there often exist additional constraints in the decision-making process such as hard-constraint like safety or soft-constraint like cost. To this end,  there have been exciting recent advances in the theoretical analysis of constrained kernelized bandits. In particular,~\citep{sui2015safe,berkenkamp2016bayesian,sui2018stagewise} propose algorithms with convergence guarantees, while~\citep{amani2020regret}, to the best our knowledge, is the first work that establishes regret bounds for their developed algorithm, although under the Bayesian-type\footnote{In the Bayesian-type KB, $f$ is assumed to be a sample from a Gaussian process and the observation noise is Gaussian. In contrast, in our considered frequentist-type KB, $f$ is a fixed function in an RKHS and the noise can be any sub-Gaussian. In general, a better regret bound can be achieved in the easier Bayesian-type KB setting~\citep{srinivas2009gaussian}.} setting.
These algorithms mainly focus on KB with a \emph{hard} constraint such as safety, i.e., the selected action in \emph{each round} needs to satisfy the constraint with a high probability. As a result, compared to the unconstrained case, additional computation is often required to construct a \emph{safe} action set in each round, which not only incurs additional complexity burdens, but often leads to conservative performance.

\textbf{Motivations.} In practice, there are also many applications that involve \emph{soft} constraints that may be violated in any round. The goal is to maximize the total reward while minimizing the total constraint violations. To give a concrete example, let us consider resource configuration in cloud computing platforms where the objective is to minimize the cost while guaranteeing that the latency is below a threshold, e.g., $95\%$ percentile latency. In this case, the latency of a job could be above the threshold as long as the fraction of violations is small, e.g., less than $5\%$. Another example is throughput maximization under energy constraints in wireless communications where energy consumption constraint is often a soft cumulative one. In both examples, one fundamental question is {whether the nature of soft constraints can be utilized to design constrained KB (CKB) algorithms with the same complexity as the unconstrained case while attaining a better reward performance compared to the hard constraints}. Furthermore, existing provably efficient algorithms~\citep{sui2015safe,berkenkamp2016bayesian,sui2018stagewise,amani2020regret} largely build on upper confidence bound (UCB) exploration, which often has inferior empirical performance compared to Thompson sampling (TS) exploration. 
Hence, another key question is {whether one can design provably efficient CKB algorithms with general explorations}. In summary, the following fundamental theoretical question remains open:
\begin{center}
    \emph{Can a finer complexity-regret-constraint trade-off be attained in CKB under general explorations? }
\end{center}

\textbf{Contributions.} In this paper, we take a systematic approach to affirmatively answer the above fundamental question. In particular, we tackle the complexity-regret-constraint trade-off by formulating KB under soft constraints as a stochastic bandit problem where the objective is to maximize the cumulative reward while minimizing the cumulative constraint violations and maintaining the same computation complexity as in the unconstrained case. Our detailed contributions can be summarized as follows.

\begin{itemize}
    \item We develop a unified framework for CKB based on primal-dual optimization, which can guarantee both sublinear reward regret and sublinear total constraint violation under a class of general exploration strategies, including UCB, TS, and new effective ones (e.g., random exploration) under the same complexity as the unconstrained case. We also show that by introducing slackness in the dual update, one can trade regret to achieve bounded or even zero constraint violation. 
    This framework builds upon a novel sufficient condition, which not only facilitates the design of new CKB algorithms but provides a unified view in the performance analysis.
    \item We demonstrate the superior performance of our proposed algorithms via numerical experiments based on both synthetic and real-world data. In addition, we discuss the benefits of our algorithms in terms of various practical considerations such as low complexity, scalability, robustness, and flexibility.
    \item Finally, we provide the first detailed comparison between two popular methods for analyzing constrained bandits and MDPs in general. Specifically, The first one is based on convex optimization tool as in~\citep{efroni2020exploration,ding2021provably}, which is also the inspiration for our paper. The other one is based on Lyapunov-drift argument as in~\citep{liu2021efficient,wei2021provably,liu2021learning}. We discuss the key difference in terms of regret and constraint violation analysis in these two methods and highlight the subtlety in applying a standard queueing technique (i.e., Hajek lemma~\citep{hajek1982hitting}) to bound the constraint violation in the second method. We believe this provides a clear picture on the methodology, which is of independent interest to the communities.

\end{itemize}

\subsection{Related Work}
In the special cases of KB, such as multi-armed bandits (MAB) and linear bandits (e.g., KB with a linear kernel), there is a large body of work on bandits with different types of constraints, including knapsack bandits~\citep{agrawal2016linear,badanidiyuru2013bandits,wu2015algorithms}, conservative bandits~\citep{wu2016conservative,kazerouni2016conservative,garcelon2020improved}, bandits with fairness constraints~\citep{chen2020fair,li2019combinatorial}, bandits with hard safety constraints~\citep{amani2019linear,pacchiano2021stochastic,moradipari2019safe}, and bandits with cumulative soft constraints~\citep{liu2020pond,liu2021efficient}. Among them, the bandit setting with cumulative soft constraints is the closest to ours in that the goal is also to minimize the cumulative constraint violation. 
In particular,~\citep{liu2021efficient} considers linear bandits under UCB exploration and a zero constraint violation is attained via the Lyapunov-drift method. However, it is unclear how to generalize it to handle general exploration strategies; see a further discussion in Section~\ref{sec:dis}.


Broadly speaking, our work is also related to reinforcement learning (RL) with soft constraints, i.e., constrained MDPs.
In particular, our analysis is inspired by those on constrained MDPs~\citep{efroni2020exploration,ding2021provably} (which is another popular method to handle constrained bandits and MDPs via convex optimization tools), but has significant differences. First, in those works, the constraint violation is $\widetilde{O}(\sqrt{T})$. In contrast, ours can attain bounded and even zero constraint violations by introducing the slackness in the dual update. Second, they only consider UCB-type exploration, but our algorithms can be equipped with various exploration strategies (including UCB), thanks to our general sufficient condition. 
Third, they focus on either tabular or linear 
function approximation settings. In contrast, both objective and constraint functions we consider can be \emph{nonlinear}. 
There are also recent works on constrained MDPs that claim to achieve bounded or zero constraint violation~\citep{wei2021provably,liu2021learning} based on the Lyapunov-drift method. However, as in the bandit case, it is unclear how to generalize it to handle general explorations beyond UCB. 

Finally, we remark that our work is mainly a theory-guided study. In a more practical area of KB, i.e., Bayesian optimization (BO), there have been many BO algorithms developed for the constrained setting; see ~\citep{eriksson2021scalable,gelbart2014bayesian,hernandez2016general} and the references therein. Although these algorithms have demonstrated good performance in various practical settings, their theoretical performance guarantees are still unclear.

\section{Problem Formulation and Preliminaries}

We consider a stochastic bandit optimization problem with \emph{soft} constraints, i.e., $\max_{x \in \mathcal{X}} f(x)$ subject to $g(x) \le 0$, where $\mathcal{X} \subset \mathbb{R}^d$ and both $f:\mathcal{X} \to \mathbb{R}$ and $g:\mathcal{X} \to \mathbb{R}$ are unknown functions\footnote{Our main results can be readily generalized to the multi-constraint case with a properly chosen norm.}. 
In particular, in each round $t \in \{1,2,\dots,T\}$, a learning agent chooses an action $x_t \in \mathcal{X}$ and receives a bandit reward feedback $r_t = f(x_t) + \eta_t$, where $\eta_t$ is a zero-mean noise. The learning agent also observes a bandit constraint feedback $c_t = g(x_t) + \xi_t$, where $\xi_t$ is a zero-mean noise. To capture the feature of soft constraints, the goal here is to maximize the cumulative reward (i.e., $\sum_{t=1}^T f(x_t)$) while minimizing the cumulative constraint violation (i.e., $\sum_{t=1}^T g(x_t)$) throughout the learning process.


\textbf{Learning Problem.} 
Define cumulative regret and constraint violation as follows:
\begin{align*}
    &\mathcal{R}(T):=Tf(x^*) - \sum_{t=1}^T f(x_t), \quad \mathcal{V}(T):=\left[\sum_{t=1}^T g(x_t)\right]_+,
\end{align*}
where $x^* := \argmax_{\{x \in \mathcal{X}: {g}(x) \le 0\} }f(x)$ and $[\cdot]_+:= \max\{\cdot,0\}$.
The goal is to achieve both sublinear regret and sublinear constraint violation.
In fact, we will establish bounds on the following stronger version of regret. Specifically, let $\pi$ be a probability distribution over the set of actions $\mathcal{X}$, and let $\mathbb{E}_{\pi}[f(x)]:= \int_{x \in \mathcal{X}} f(x)\pi(x)\,dx$ and $\mathbb{E}_{\pi}[g(x)]:= \int_{x \in \mathcal{X}} g(x)\pi(x)\,dx$. We compare our achieved reward with the following optimization problem: $\max_{\pi} \{ \mathbb{E}_{\pi}[f(x)]:\mathbb{E}_{\pi}\left[g(x)\right] \le 0\}$ where both $f$ and $g$ are known, and $\pi^{*}$ is its optimal solution. Now, a stronger regret is defined as 
\begin{align}
\label{eq:def_regret}
    \mathcal{R}_+(T) := T \mathbb{E}_{\pi^*}\left[f(x)\right] -  \sum_{t=1}^{T} {f(x_t)}.
\end{align}
Clearly, we have $\mathcal{R}(T) \le \mathcal{R}_+(T)$.
Throughout the paper, we assume the following commonly used condition in constrained optimization; see also~\citep{liu2021efficient,yu2017online,efroni2020exploration}. 
\begin{assumption}[Slater's condition]
\label{ass:slater_bandit}
There is a constant $\delta >0$ such that there exists a probability distribution $\pi_0$ that satisfies $\mathbb{E}_{\pi_0}\left[g(x)\right]  \le  - \delta$. Without loss of generality, we assume $\delta \le 1$.
\end{assumption}
This is a quite mild assumption since it only requires that one can find a probability distribution over the set of actions under which the expected cost is less than a strictly negative value. This is in sharp constraint to existing KB algorithms for hard constraints that typically require the existence of an initial safe action~\citep{sui2018stagewise,amani2020regret}.


In this paper, we consider the  frequentist-type regularity assumption that is typically used in uncontrained KB works (e.g.,~\citep{chowdhury2017kernelized,srinivas2009gaussian}). Specifically, we assume that $f$ is a fixed function in an RKHS with a bounded norm. In particular, the RKHS for $f$ is denoted by $\mathcal{H}_{k}$, which is completely determined by the corresponding kernel function $k: \mathcal{X} \times \mathcal{X} \to \Real$. Any function $h \in \mathcal{H}_{k}$ satisfies the \emph{reproducing property}: $h(x) = \inner{h}{k(\cdot,x)}_{\mathcal{H}_k}$, where $\inner{\cdot}{\cdot}_{\mathcal{H}_k}$ is the inner product defined on $\mathcal{H}_{k}$. Similarly, for the unknown constraint function $g$, we assume that $g$ is a fixed function in the RKHS defined by a kernel function $\widetilde{k}$, and the RKHS for $g$ is denoted by $\mathcal{H}_{\widetilde{k}}$.
We assume that the following boundedness property holds throughout the paper.

\begin{assumption}[Boundedness]
\label{ass:boundf}
We assume that $\norm{f}_{\mathcal{H}_k} \le B$ and $k(x,x) \le 1$ for any $x \in \mathcal{X}$ and that the noise $\eta_t$ is \emph{i.i.d.} $R$-sub-Gaussian. Similarly, we assume that $\norm{g}_{{\mathcal{H}_{\widetilde{k}}}} \le G$ and $\widetilde{k}(x,x) \le 1$ for any $x \in \mathcal{X}$ and that the noise $\xi_t$ is  \emph{i.i.d.} $\widetilde{R}$-sub-Gaussian.
\end{assumption}

\textbf{Gaussian Process Surrogate Model.} We use a Gaussian process (GP), denoted by $\mathcal{GP}(0,k(\cdot,\cdot))$, as a prior for the unknown function $f$, and a Gaussian likelihood model for the noise variables $\eta_t$, which are drawn from $\mathcal{N}(0,\lambda)$ and are independent across $t$. Note that this GP surrogate model is used for algorithm design only; it does not change the fact that $f$ is a fixed function in $\mathcal{H}_{k}$ and that the noise $\eta_t$ can be sub-Gaussian (i.e., an \emph{agnostic} setting~\citep{srinivas2009gaussian}). 

Let $[t] :=\{1,2,\ldots, t\}$. Conditioned on a set of observations ${H}_{t} = \{(x_s,{r}_s), s\in[t]\}$, by the properties of GP~\citep{rasmussen2003gaussian}, the posterior distribution for $f$ is $\mathcal{GP}(\mu_{t}(\cdot),k_{t}(\cdot,\cdot))$, where 
\begin{align}
    &\mu_{t}(x) := k_{t}(x)^T(K_{t} + \lambda I)^{-1}R_t, \label{eq:mu}\\
    &k_{t}(x,x') := k(x,x')-k_{t}(x)^T(K_{t} + \lambda I)^{-1}k_{t}(x') \label{eq:sigma},
\end{align}
in which $k_{t}(x):=[k(x_1,x),\ldots,k(x_t,x)]^T$, $K_{t}:=[k(x_u,x_v)]_{u,v \in [t]}$, and $R_t$ is the (noisy) reward vector $ [r_1,r_2,\ldots,r_t]^T$. In particular, we also define $\sigma_{t}^2(x) := k_{t}(x,x)$.
Let $K_A := [k(x,x')]_{x,x'\in A}$ for $A \subseteq \mathcal{X}$. We define the maximum
information gain as  $\gamma_t(k,\mathcal{X}) := \max_{A \subseteq \mathcal{X}: |A| = t} \frac{1}{2}\ln |I_t + {\lambda}^{-1}K_A |$ where $I_t$ is the $t \times t$ identity matrix. 
	The maximum information gain plays a key role in the regret bounds of GP-based algorithms. While $\gamma_t(k,\mathcal{X})$ depends on the kernel $k$ and domain $\mathcal{X}$, we simply use $\gamma_t$ whenever the context is clear. For instance, if $\mathcal{X}$ is compact and convex with dimension $d$, then we have $\gamma_t = O((\ln t)^{d+1})$ for squared exponential kernel $k_\text{SE}$, $\gamma_t=O(t^{\frac{d(d+1)}{2\nu+d(d+1)}}\ln t)$ (where $\nu$ is a hyperparameter) for \text{Mat\'ern} kernel $k_{\text{Mat\'ern}}$, and $\gamma_t=O(d\ln t)$ for linear kernel~\citep{srinivas2009gaussian}. Similarly, the learning agent also uses a GP surrogate model for $g$, i.e., a GP prior $\mathcal{GP}(0,\widetilde{k}(\cdot,\cdot))$ and a Gaussian noise $\mathcal{N}(0,\widetilde{\lambda})$. Conditioned on a set of observations $\widetilde{H}_{t} = \{(x_s,{c}_s), s\in[t] \}$, the posterior distribution for $g$ is $\mathcal{GP}(\widetilde{\mu}_{t}(\cdot),\widetilde{k}_{t}(\cdot,\cdot))$, where $\widetilde{\mu_t}$ and $\widetilde{k}_t$ are computed in the same way as $\mu_t(\cdot)$ and $k_t(\cdot,\cdot)$.

\section{A Unified Framework for Constrained Kernelized Bandits}
\label{sec:bandit}
In this section, leveraging primal-dual optimization, we propose a unified framework for both algorithm design and performance analysis. In particular, we first propose a ``master'' algorithm called CKB (\underline{c}onstrained KB), which can be equipped with very general exploration strategies.
Then, we develop a novel sufficient condition, which not only provides a unified analysis of regret and constraint violation, but also facilitates the design of new exploration strategies (and hence new CKB algorithms) with rigorous performance guarantees.

\begin{algorithm}[t!]
\caption{CKB Algorithm}
\label{alg:CKB}
\DontPrintSemicolon
\KwIn{$V$, $\rho$, $\phi_1 = 0$, $\mu_{0}(x) =  \widetilde{\mu}_0(x) = 0$, $\sigma_{0}(x) =  \widetilde{\sigma}_0(x) = 1, \forall x$, exploration strategies $\mathcal{A}_f$ and $\mathcal{A}_g$ }
\For{$t\!=\!1,2,\dots, T$}{
Based on posterior models, generate $f_t$ and $g_t$ using $\mathcal{A}_f$ and $\mathcal{A}_g$, respectively\;
Truncate $f_t$ as $\bar{f}_{t}(x) = \text{Proj}_{[-B,B]} f_t(x)$ \;
Truncate $g_t$ as $\bar{g}_{t}(x) =\text{Proj}_{[-G,G]} g_t(x)$\;
{Pseudo-acquisition function:} $\hat{z}_{\phi_t}(x) = \bar{f}_{t}(x) - \phi_t \bar{g}_t(x)$\;
Choose primal action $x_t = \argmax_{x\in  \mathcal{X}} \hat{z}_{\phi_t}(x)$; observe $r_t$ and $c_t$\;
{Update dual variable:} $\phi_{t+1} = \text{Proj}_{[0,\rho]}\left[\phi_t +  \frac{1}{V}\bar{g}_{t}(x_t)  \right]$\;
Posterior model: update $(\mu_{t}, \sigma_{t})$ and  $(\widetilde{\mu}_t, \widetilde{\sigma}_{t})$ via GP regression using new data $(x_t,r_t,c_t)$ \;
}
\end{algorithm}

\textbf{Algorithm.} We first explain our ``master'' algorithm CKB in Algorithm~\ref{alg:CKB}, which is based on primal-dual optimization. Let the Lagrangian of the baseline problem $\max_{\pi} \{ \mathbb{E}_{\pi}[f(x)]: \mathbb{E}_{\pi}\left[g(x)\right] \le 0\}$ be $\mathcal{L}(\pi,\phi):= \mathbb{E}_{\pi}\left[f(x)\right] - \phi \mathbb{E}_{\pi}\left[g(x)\right]$ and the associated dual problem is defined as $\mathcal{D}(\phi):= \max_{\pi} \mathcal{L}(\pi,\phi)$ with the optimal dual variable being $\phi^*:=\argmin_{\phi\ge 0} \mathcal{D}(\phi)$. Note that since both $f$ and $g$ are unknown, the agent has to first generate estimates of them (i.e., $f_t$ and $g_t$, respectively) based on exploration strategies $\mathcal{A}_f$ and $\mathcal{A}_g$, which capture the tradeoff between exploration and exploitation (line 2). Then, both estimates will be truncated according to the range of $f$ and $g$, respectively (lines 3-4) (where $\text{Proj}$ is the projection operator). The truncation is necessary for our analysis, but it does not impact the regret bound since it will not lead to loss of useful information. Then, lines 5-6 correspond to the primal optimization step that approximates $\mathcal{D}(\phi_t)$ (i.e., approximate $\mathcal{L}$ by $\bar{\mathcal{L}}$  with $f$ and $g$ replaced by $\bar{f}_t$ and $\bar{g}_t$). The reason behind line 6 is that one of the optimal solutions for $\max_{\pi} \bar{\mathcal{L}} (\pi,\phi_t)$ is simply $\argmax_x (\bar{f}_t (x) - \phi_t \bar{g}_t(x))$.
Then, line 7 is the dual update that minimizes $\mathcal{D}(\phi_t)$ with respect to $\phi$ by taking a projected gradient step with $1/V$ being the step size. The parameter $\rho$ is chosen to be larger than the optimal dual variable $\phi^*$, and hence the projected interval $[0,\rho]$ includes the optimal dual variable. This can be achieved since the optimal dual variable is bounded under Slater's condition, and in particular, we have $\phi^* \le (\mathbb{E}_{\pi^*}\left[f(x)\right] - \mathbb{E}_{\pi_0}\left[f(x)\right])/\delta$ by~\citep[Theorem 8.42]{beck2017first}. Finally, line 8 is the posterior update via standard GP regression for both $f$ and $g$ as computed in~\eqref{eq:mu} and~\eqref{eq:sigma}
with $\sigma_{t}^2(x) = k_{t}(x,x)$ and $\widetilde{\sigma}_{t}^2(x) = \widetilde{k}_{t}(x,x)$.

\begin{remark}[Computational complexity]
CKB enjoys the same computational complexity as the unconstrained case (e.g.,~\citep{chowdhury2017kernelized}) since the additional dual update is a simple projection and the primal optimization keeps the same flavor as the unconstrained case, i.e., without constructing a specific safe set as in existing constrained KB algorithms designed for hard constraints.
\end{remark}

We call CKB  a ``master'' algorithm as it allows us to employ different exploration strategies (or called \emph{acquisition functions}) (i.e., $\mathcal{A}_f$ and $\mathcal{A}_g$). 
Therefore, one fundamental question is: \emph{How to design efficient exploration strategies such that favorable performance can be guaranteed?}
In the following, we take a two-step procedure to address this question. We first combine UCB-type exploration with CKB to gain useful insights. 
This, in turn, will facilitate the development of a novel sufficient condition, which not only is satisfied under very general exploration strategies, but also enables a unified analytical framework for showing both sublinear regret and sublinear constraint violation.

Before that, we first introduce standard UCB and TS explorations under GP  as in~\citep{chowdhury2017kernelized}.

\begin{definition}[GP-UCB and GP-TS Explorations]
\label{def:GP-exp}
Suppose the posterior distribution for a black-box function $h$ in round $t$ is given by $\mathcal{GP}(\hat{\mu}_{t-1}(\cdot), \hat{k}_{t-1}(\cdot,\cdot))$ and  $\hat{\beta}_t$ is a time-varying sequence. 
\begin{enumerate}
    \item The estimate of $h$ in round $t$ under GP-UCB exploration is  $h_t(\cdot) = \hat{\mu}_{t-1}(\cdot) + \hat{\beta}_t \hat{\sigma}_{t-1}(\cdot)$, where $\hat{\sigma}_{t-1}^2(x) := \hat{k}_{t-1}(x,x)$ for all $x \in \mathcal{X}$.
    \item The estimate of $h$ in round $t$ under GP-TS exploration is $h_t(\cdot) \sim  \mathcal{GP}(\hat{\mu}_{t-1}(\cdot), \hat{\beta}_t^2 \hat{k}_{t-1}(\cdot,\cdot))$.
\end{enumerate}
\end{definition}




\subsection{Warm Up: CKB with GP-UCB Exploration}
In this section, we instantiate CKB with GP-UCB exploration called {\texttt{CKB-UCB}}, as a warm-up. In particular, in \texttt{CKB-UCB},  $\mathcal{A}_f$ is a GP-UCB exploration (see Definition~\ref{def:GP-exp}) with a positive $\hat{\beta}_t$ sequence (i.e., optimistic with respect to reward), and $\mathcal{A}_g$ is a GP-UCB exploration  with a negative $\hat{\beta}_t$ sequence (i.e., optimistic with respect to cost). This instantiation enjoys the following performance guarantee.


\begin{theorem}
\label{thm:bandit_UCB}
Suppose $\rho \ge 4B/\delta$, $V = G\sqrt{T}/\rho$, $\mathcal{A}_f$ is a GP-UCB exploration with $\hat{\beta}_t = \beta_t  := B + R\sqrt{2(\gamma_{t-1} + 1 + \ln(2/\alpha))}$, and $\mathcal{A}_g$ is a GP-UCB exploration with $\hat{\beta}_t = -\widetilde{\beta}_t  := -(G + R\sqrt{2(\widetilde{\gamma}_{t-1} + 1 + \ln(2/\alpha))})$.  Under Slater's condition in Assumption~\ref{ass:slater_bandit} and regularity assumptions in Assumption~\ref{ass:boundf}, {\texttt{CKB-UCB}} achieves the following bounds simultaneously with probability at least $1-\alpha$ for any $\alpha \in (0,1)$:
\begin{align*}
    &\mathcal{R}_+(T) = O\left(B\sqrt{T\gamma_T} + \sqrt{T\gamma_T(\gamma_T + \ln (2/\alpha))} + \rho G \sqrt{T}\right),\\
    &\mathcal{V}(T) = O\left((1+\frac{1}{\rho})\left( C\sqrt{T\hat{\gamma}_T} + \sqrt{T\hat{\gamma}_T (\hat{\gamma}_T + 2\ln(2/\alpha))}\right) + G\sqrt{T} \right),
\end{align*}
where $C := \max\{B,G\}$ and $\hat{\gamma}_T := \max\{\gamma_T,\widetilde{\gamma}_T\}$.
\end{theorem}
\begin{remark}
The (reward) regret here is the stronger version, i.e., $\mathcal{R}_+(T)$. Compared to the unconstrained case, the regret bound has an additional term $\rho G\sqrt{T}$, which roughly captures the impact of the constraint. As in the unconstrained case, one can plug in different $\gamma_T$ and $\widetilde{\gamma}_T$ to see that both regret and constraint violation are sublinear for commonly used kernels. 
For example, for a linear kernel, both $\gamma_T$ and $\widetilde{\gamma}_T$ are on the order of $d \ln(T)$. 
Finally, the standard ``doubling trick'' can be used to design an anytime algorithm {(i.e., without the knowledge of $T$)} with regret and constraint violation bounds of the same order.
\end{remark}

\textbf{Proof Sketch of Theorem~\ref{thm:bandit_UCB}.} 
We first obtain the following key decomposition that holds for any $\phi \in [0,\rho]$:  $\mathcal{R}_+(T) + \phi\sum_{t=1}^T g(x_t) \le \mathcal{T}_1 + \mathcal{T}_2 + \frac{V}{2}\phi^2  + \frac{1}{2V}TG^2$, where 
\begin{align}
    \mathcal{T}_1 := &\sum_{t=1}^{T} (\mathbb{E}_{\pi^*}\left[f(x)\right]- \phi_t \mathbb{E}_{\pi^*}\left[g(x)\right]) - \sum_{t=1}^{T} (\bar{f}_t(x_t) - \phi_t \bar{g}_t(x_t))\label{eq:t1},\\
    \mathcal{T}_2 := &\sum_{t=1}^{T} (\bar{f}_t(x_t) - f(x_t))
     +\phi \sum_{t=1}^{T} (g(x_t) - \bar{g}_t(x_t))\label{eq:t2}.
\end{align}
This is achieved by utilizing the dual variable update and some necessary algebra. This bound will be the cornerstone for the analysis of both regret and constraint violation. Note that $\mathcal{T}_1 + \mathcal{T}_2$ is similar to the standard regret decomposition, with an incorporation of the constraint function weighted by $\phi_t$ (or $\phi$). Assume that we already have a bound on it, i.e., $\mathcal{T}_1 + \mathcal{T}_2 \le \chi(T,\phi)$ with high probability, and $\chi(T,\phi)$ is an increasing function of $\phi$. This leads to the following inequality (with $V = G\sqrt{T}/\rho$) for all $\phi \in [0,\rho]$: 
\begin{align}
\label{eq:decomp}
    \mathcal{R}_+(T) + \phi\sum_{t=1}^T g(x_t) 
    \le \chi(T,\phi) + \frac{\phi^2 G\sqrt{T}}{2\rho} + \frac{\rho G\sqrt{T}}{2}.
\end{align}
Then, the regret bound can be obtained by choosing $\phi = 0$ in~\eqref{eq:decomp}, and hence, $\mathcal{R}_+(T) = O(\chi(T,0) + \rho G\sqrt{T})$. Inspired by~\citep{efroni2020exploration}, we will resort to tools from constrained convex optimization to obtain the bound on $\mathcal{V}(T)$. First, we have $\frac{1}{T}\sum_{t=1}^T f(x_t) =   \mathbb{E}_{\pi'}\left[f(x)\right]$ and $\frac{1}{T}\sum_{t=1}^T g(x_t) = \mathbb{E}_{\pi'}\left[g(x)\right]$ for some probability measure $\pi'$ by the convexity of probability measure. Then, we have 
\begin{align*}
  \mathbb{E}_{\pi^*}\left[f(x)\right] - \mathbb{E}_{\pi'}\left[f(x)\right] + \rho \left[\mathbb{E}_{\pi'}\left[g(x)\right]\right]_+  = \frac{1}{T} \mathcal{R}_+(T) + \frac{1}{T}\phi\sum_{t=1}^T g(x_t) \le \frac{\chi(T,\rho) + \rho G\sqrt{T}}{T},
\end{align*}
where the equality holds by choosing $\phi = \rho$ if $\sum_{t=1}^Tg(x_t) \ge 0$, and otherwise $\phi=0$, and the inequality holds by bounding RHS of~\eqref{eq:decomp} with $\phi = \rho$ since~\eqref{eq:decomp} holds for all $\phi \in [0,\rho]$ and $\chi(T,\phi)$ is increasing in $\phi$. Then, based on the above result, we can apply the tool from constrained convex optimization (cf.~\citep[Theorem 3.60]{beck2017first}) to obtain $\mathcal{V}(T) \le \frac{1}{\rho}\chi(T,\rho) + G\sqrt{T}$. The reason why we can apply this result is that $\mathbb{E}_{\pi}\left[h(x)\right]$ for any fixed $h$ is a linear function with respect to $\pi$ (and is thus convex). 
Finally, it remains to find $\chi(T,\phi)$ that bounds $\mathcal{T}_1 + \mathcal{T}_2$. This can be achieved by using results from unconstrained GP-UCB algorithm (cf.~\citep{chowdhury2017kernelized}). In particular, we have $\mathcal{T}_1 \le 0$ and $\mathcal{T}_2 \le 2\beta_T\sum_{t=1}^T\sigma_{t-1}(x_t) + 2\phi \widetilde{\beta}_T \sum_{t=1}^T \widetilde{\sigma}_{t-1}(x_t) = O(\beta_T\sqrt{\gamma_T T} +\phi \widetilde{\beta}_T \sqrt{\widetilde{\gamma}_T T})$. Then, we have $\chi(T,\phi) =  O(\beta_T\sqrt{\gamma_T T} +\phi \widetilde{\beta}_T \sqrt{\widetilde{\gamma}_T T})$. Finally, plugging $\chi(T,0)$  and $\chi(T,\rho)$ into $\mathcal{R}_+(T)$ and $\mathcal{V}(T)$ yields the bounds on regret and on constraint violation, respectively, which completes the proof.
\qed


\subsection{A Sufficient Condition for Provably Efficient Explorations}
The above analysis reveals that the key step in obtaining sublinear performance guarantees of Algorithm~\ref{alg:CKB} is to find a sublinear bound on $\chi(T,\phi)$ that bounds $\mathcal{T}_1 + \mathcal{T}_2$. Motivated by this, in this section, we will establish a sufficient condition on the exploration strategies (i.e., $\mathcal{A}_f$ and $\mathcal{A}_g$), which guarantees a sublinear $\chi(T,\phi)$ and hence sublinear regret and sublinear constraint violation. In particular, we show that existing strategies such as GP-UCB and GP-TS both satisfy the sufficient condition. More importantly, this sufficient condition also leads to the development of new exploration strategies (such as random exploration).


We first present the intuition behind the key components of the sufficient condition. Inspired by~\citep{kveton2019perturbed}, we mainly focus on the following three nice events to bound $\mathcal{T}_1 + \mathcal{T}_2$ in~\eqref{eq:t1}-\eqref{eq:t2}:
\begin{align}
    &E^{est}:= \{E^{est}_f(x,t) \cap E^{est}_g(x,t); \forall (x,t)\},\nonumber\\
    &E^{conc}_{t} := \{E^{conc}_{f,t}(x) \cap  E^{conc}_{g,t}(x); \forall x\},\nonumber\\
    &E^{anti}_{t} := E^{anti}_{f,t} \cap E^{anti}_{g,t},\nonumber
\end{align}
where 
\begin{align*}
    &E^{est}_f(x,t):= |f(x) -\mu_{t-1}(x) | \le c_{f,t}^{(1)}\sigma_{t-1}(x), E^{est}_g(x,t):= |g(x) -\widetilde{\mu}_{t-1}(x) | \le c_{f,t}^{(1)}\widetilde{\sigma}_{t-1}(x),\\
    &E^{conc}_{f,t}(x):=|f_t(x) - \mu_{t-1}(x)| \le c_{f,t}^{(2)}\sigma_{t-1}(x),E^{conc}_{g,t}(x):=|g_t(x) - \widetilde{\mu}_{t-1}(x)| \le c_{g,t}^{(2)}\widetilde{\sigma}_{t-1}(x),\\
    &E^{anti}_{f,t}:=\mathbb{E}_{\pi^*}\left[ f_t(x) - \mu_{t-1}(x)\right] \ge c_{f,t}^{(1)}\mathbb{E}_{\pi^*}[\sigma_{t-1}(x)], E^{anti}_{g,t}:=\mathbb{E}_{\pi^*}\left[ g_t(x) - \widetilde{\mu}_{t-1}(x)\right] \le -c_{g,t}^{(1)}\mathbb{E}_{\pi^*}[\widetilde{\sigma}_{t-1}(x)].
\end{align*}
Suppose that events $E^{est}$ and $E^{conc}_{t}$ hold with high probability. Then, it is easy to see that the estimates are close to the true functions, and hence, one can derive a bound on $\mathcal{T}_2$ in~\eqref{eq:t2}. Now, suppose that events $E^{est}$ and $E^{anti}_{t}$ hold with some positive probability. Then, one can see that the estimates are optimistic compared to the true functions when evaluated at the optimal points. This probabilistic optimism is the key to bounding $\mathcal{T}_1$ in~\eqref{eq:t1}. Note that GP-UCB exploration is optimistic with probability one  by definition (see Definition~\ref{def:GP-exp}), and hence, $\mathcal{T}_1 \le 0$ always holds.

Define the filtration $\mathcal{F}_{t}$ as all the history up to the end of round $t$. Let $\ext{\cdot} := \mathbb{E}\left[\cdot | \mathcal{F}_{t-1}\right]$ and $\pt{\cdot} := \mathbb{P}\left[\cdot | \mathcal{F}_{t-1}\right]$. We are now ready to present the sufficient condition for exploration strategies. 


\begin{assumption}[Sufficient Condition]
\label{ass:sufficient}
The sufficient condition includes two parts:

(1) (Probability condition) $\mathbb{P}\left(E^{est}\right) \ge 1-p_{1}$, $\pt{E_{t}^{conc}} \ge 1-p_{2,t}$, and $\pt{E_{t}^{anti}} \ge p_{3} >0$ for some time-dependent sequences $c_{f,t}^{(1)}, c_{g,t}^{(1)}, c_{f,t}^{(2)}$, and $c_{g,t}^{(2)}$. 

(2) (Boundedness condition) (i) There exists a positive probability $p_4$ such that $1+\frac{2}{(p_3 - p_{2,t})} \le 1/p_4$ for all $t$; (ii) $c_{f,t}^{(1)} + c_{f,t}^{(2)} \le c_f(T)$ and $c_{g,t}^{(1)} + c_{g,t}^{(2)} \le c_g(T)$ for all $t$; (iii) $\sum_{t=1}^T p_{2,t} \le C'$ for some constant $C'$.
    
\begin{remark}
The above sufficient condition generalizes existing similar results~\citep{kveton2019perturbed,vaswani2019old,kim2020randomized} in several aspects. First, existing works mainly focus on the MAB or linear bandit settings, which are special cases of our KB setting (e.g., choosing a linear kernel leads to linear bandit). Second, while existing works only establish bounds on the expected regret, we aim to establish a high-probability bound. As a result, we need the additional boundedness condition, which, however, is simply for technical reasons. 
Third, in contrast to existing works that consider the unconstrained case only, we consider the constrained case, which is more challenging. Specifically, it requires that $E_t^{anti}$ holds under policy $\pi^*$ (i.e., the expectation over $\pi^*$) rather than under a single optimal action $x^*$.

\end{remark}
\end{assumption}

With the above sufficient condition, we have the following general performance bounds. 
\begin{theorem}
\label{thm:rand_exp}
Suppose $\rho \ge 4B/\delta$ and $V = G\sqrt{T}/\rho$. Let $\kappa:= B+\rho G$. Assume that CKB is equipped with exploration strategies that satisfy the sufficient condition in Assumption~\ref{ass:sufficient}.  Then, under Slater's condition in Assumption~\ref{ass:slater_bandit} and regularity assumptions in Assumption~\ref{ass:boundf}, CKB achieves the following bounds on regret and constraint violation with probability at least $1-\alpha-p_1$ for any $\alpha \in (0,1)$:
\begin{align*}
    \mathcal{R}_+(T) &= O\left(\frac{1}{p_4}c_f(T)\sqrt{T\gamma_T} + \frac{1}{p_4}\rho c_g(T) \sqrt{T\widetilde{\gamma}_T} +\rho G\sqrt{T} + \kappa \frac{c_f(T)+\rho c_g(T)}{p_4}\sqrt{2T\ln(1/\alpha)}\right),\\
    \mathcal{V}(T) &= O\left(\frac{1}{\rho p_4}c_f(T)\sqrt{T\gamma_T} + \frac{1}{p_4}c_g(T) \sqrt{T\widetilde{\gamma}_T} +G\sqrt{T} + \kappa \frac{c_f(T)+ \rho c_g(T)}{\rho p_4}\sqrt{2T\ln(1/\alpha)}\right).\\
\end{align*}
\end{theorem}
In the following, we will show that Theorem~\ref{thm:rand_exp} provides a unified view of the performance for various CKB algorithms. 



\subsection{CKB Algorithms Under Various Explorations}

The sufficient condition not only recovers previous exploration strategies, but also reveals new ones. First, we remark that the first event $E^{est}$ in the probability condition of Assumption~\ref{ass:sufficient} can be easily obtained by standard GP concentration result. That is,  by~\citep[Theorem 2]{chowdhury2017kernelized},  we have $\mathbb{P}(\forall x, t, |f(x) - \mu_{t-1}(x)| \le \beta_t \sigma_{t-1}(x)) \ge 1-\alpha_f$ for any $\alpha_f \in (0,1)$, where $\beta_t = B + R\sqrt{2(\gamma_{t-1} + 1 + \ln(1/\alpha_f))}$ (similar for $g$). 
Thus, we have $\mathbb{P}\left(E^{est}\right) \ge 1-p_1$ with $p_{1} = \alpha_f + \alpha_g$, $c_{f,t}^{(1)} = \beta_t$, and $c_{g,t}^{(1)} = \widetilde{\beta_t} = B + R\sqrt{2(\widetilde{\gamma}_{t-1} + 1 + \ln(1/\alpha_g))}$.
Thus, we only need to check probability condition for the remaining two events and the boundedness condition under different exploration methods. 

It is expected that GP-UCB exploration satisfies the sufficient condition and hence our \texttt{CKB-UCB} also has performance guarantees given by Theorem~\ref{thm:rand_exp}. In particular, we see that Theorem~\ref{thm:rand_exp} enjoys the same order of constraint violation as in  Theorem~\ref{thm:bandit_UCB}. The regret bound has the same order as Theorem~\ref{thm:bandit_UCB}, with an additional term due to the unified analysis (i.e., $\rho c_g(T) \sqrt{T\widetilde{\gamma}_T}$) . 
\begin{corollary}
\label{cor:ucb}
GP-UCB with $\hat{\beta}_t=\beta_t$ and $\hat{\beta}_t=-\widetilde{\beta}_t$ for $\mathcal{A}_f$ and $\mathcal{A}_g$, respectively, satisfies the sufficient condition.
\end{corollary}
\begin{proof}
By Definition~\ref{def:GP-exp}, $f_t(x) = \mu_{t-1}(x) + \beta_t\sigma_{t-1}(x)$ and $g_t(x) = \widetilde{\mu}_{t-1}(x) - \widetilde{\beta}_t\widetilde{\sigma}_{t-1}(x)$. From this, we can directly obtain that $E_t^{conc}$ and $E_t^{anti}$ hold with probability one. Moreover, the boundedness condition naturally holds.
\end{proof}

We can also show that the standard GP-TS exploration in Definition~\ref{def:GP-exp} satisfies the sufficient condition. 
Here, we mainly consider the case when $\pi^*$ concentrates on a single point, which allows us to apply the standard anti-concentration results\footnote{One can possibly utilize advanced anti-concentration results for multivariate Gaussian distributions to attain the same result for a general $\pi^*$.}.  
Thus, we can instantiate CKB with GP-TS explorations, called \texttt{CKB-TS}, that also enjoys the guarantees in Theorem~\ref{thm:rand_exp}.


\begin{corollary}
\label{cor:ts}
GP-TS with $\hat{\beta}_t=\beta_t$ and $\hat{\beta}_t=-\widetilde{\beta}_t$ for $\mathcal{A}_f$ and $\mathcal{A}_g$, respectively, satisfies the sufficient condition when $\pi^*$ concentrates on a single point.
\end{corollary}
\begin{proof}
By Definition~\ref{def:GP-exp}, we have that given the history up to the end of round $t-1$, $f_t(x) \sim \mathcal{N}(\mu_{t-1}(x), \beta_t^2 \sigma_{t-1}^2(x))$ and $g_t(x) \sim \mathcal{N}(\widetilde{\mu}_{t-1}(x), \widetilde{\beta}_t^2 \widetilde{\sigma}_{t-1}^2(x))$.
Thus, for any fixed $x \in \mathcal{X}$, by concentration of Gaussian distribution, we have $\mathbb{P}_t({|f_t(x) - \mu_{t-1}(x)| \le 2\sigma_{t-1}(x)\beta_t\sqrt{\ln t}}) \ge 1-1/t^2$, and hence, using the union bound over all $x$, we obtain $\forall x$, $\pt{E_{t,f}^{conc}(x)} \ge 1-1/t^2$ with  $c_{f,t}^{(2)} = 2\beta_t\sqrt{\ln(|\mathcal{X}| t)}$. Similarly, we have $\forall x$, $\pt{E_{t,g}^{conc}(x)} \ge 1-1/t^2$ with  $c_{g,t}^{(2)} = 2\widetilde{\beta}_t\sqrt{\ln(|\mathcal{X}| t)}$. Hence, by union bound, we have $\pt{E_{t}^{conc}} \ge 1-p_{2,t}$ with $p_{2,t} = 2/t^2$. 
Moreover, when $\pi^*$ concentrates on a single point, by standard anti-concentration result of Gaussian distribution (e.g., Lemma 8 in~\citep{chowdhury2017kernelized}), we have $\mathbb{P}_t(E_{t,f}^{anti}) \ge p$ with $p:= \frac{1}{4e\sqrt{\pi} }$.  
Similarly, we also have $\mathbb{P}_t(E_{t,g}^{anti}) \ge p$. By independent sampling of $f_t$ and $g_t$, we have $\pt{E_{t}^{anti}} \ge p_{3}$ with $p_3 = p^{2}$.
The boundedness condition holds due to $\sum_{t=1}^T p_{2,t} \le 2\sum_{t=1}^T 1/t^2 \le \pi^2/3:=C'$ and $p_4 = O(p^{2})$. 
\end{proof}

The sufficient condition also enables us to design CKB algorithms with new exploration strategies. In the following, inspired by~\citep{vaswani2019old}, we propose a new GP-based exploration strategy, which aims to strike a balance between GP-UCB and GP-TS explorations. 

\begin{definition}[RandGP-UCB Exploration]
\label{def:randgp}
Suppose that the posterior distribution for a black-box function $h$ in round $t$ is given by $\mathcal{GP}(\hat{\mu}_{t-1}(\cdot), \hat{k}_{t-1}(\cdot,\cdot))$. Then, the estimate of $h$ in round $t$ under RandGP-UCB exploration strategy is $h_t(\cdot) = \hat{\mu}_{t-1}(\cdot) + \hat{Z}_t \hat{\sigma}_{t-1}(\cdot)$, where $\hat{Z}_t \sim \hat{\mathcal{D}}$ for some distribution $\hat{\mathcal{D}}$ and $\hat{\sigma}_{t-1}^2(x) = \hat{k}_{t-1}(x,x)$ for all $x \in \mathcal{X}$. 
\end{definition}



In contrast to GP-UCB, RandGP-UCB replaces the deterministic confidence bound by a randomized one. Compared to GP-TS, RandGP-UCB uses ``coupled'' noise in the sense that all the actions share the same noise $\hat{Z}_t$ rather than ``decoupled'' and correlated noise in GP-TS. This subtle difference will not only help to eliminate the additional factor $\sqrt{\ln(|\mathcal{X}|)}$ in GP-TS due to the use of union bound, but also allow us to deal with a general $\pi^*$. One possible disadvantage of RandGP-UCB (compared to GP-TS) is that GP-TS could be {offline oracle-optimization efficient for the step in line 6 of Algorithm~\ref{alg:CKB}} while RandGP-UCB (GP-UCB also) is not, which shares the standard pattern as in linear bandits.



\begin{corollary}
\label{cor:rand}
RandGP-UCB with $\hat{\mathcal{D}}=\mathcal{N}(0,\beta_t^2)$ and $\hat{\mathcal{D}}=\mathcal{N}(0, \widetilde{\beta}_t^2)$ for $\mathcal{A}_f$ and $\mathcal{A}_g$, respectively, satisfies the sufficient condition.
\end{corollary}
\begin{proof}
By Definition~\ref{def:randgp}, $f_t(x) = \mu_{t-1}(x) + Z_t\sigma_{t-1}(x)$, where $Z_t \sim \mathcal{N}(0,\beta_t^2)$ and $g_t(x) = \widetilde{\mu}_{t-1}(x) + \widetilde{Z}_t\widetilde{\sigma}_{t-1}(x)$, where $\widetilde{Z}_t \sim \mathcal{N}(0,\widetilde{\beta}_t^2)$. By concentration of Gaussian, we have
\begin{align*}
    \mathbb{P}_t(\forall x, {|f_t(x) - \mu_{t-1}(x)| \le 2\sigma_{t-1}(x)\beta_t\sqrt{\ln t}}) \ge 1-1/t^2,
\end{align*}
thanks to the ``coupled'' noise. Hence,  we have $\pt{E_{t,f}^{conc}} \ge 1-1/t^2$ with  $c_{f,t}^{(2)} = 2\beta_t\sqrt{\ln t}$. Similarly, we have $\pt{E_{t,g}^{conc}} \ge 1-1/t^2$ with  $c_{g,t}^{(2)} = 2\widetilde{\beta}_t\sqrt{\ln t}$. Thus, by the union bound, we have $\pt{E_{t}^{conc}} \ge 1-p_{2,t}$ with $p_{2,t} = 2/t^2$. By the anti-concentration of Gaussian, we have $\mathbb{P}_t(E_{t,f}^{anti}) \ge  \mathbb{P}_t( Z_{t} \ge \beta_t) \ge p$, where $p:= \frac{1}{4e\sqrt{\pi} }$. Similarly, we have $\mathbb{P}_t(E_{t,g}^{anti}) \ge  \mathbb{P}_t( Z_{t} \le -\widetilde{\beta}_t) \ge p$. Since the noise $Z_t$ and $\widetilde{Z}_t$ are independent, we have $\pt{E_{t}^{anti}} \ge p_{3}$ with $p_3 = p^2$. Then, the boundedness condition holds due to $C' = \pi^2/3$ and $p_4 = O(p^2)$.
\end{proof}

Thus, one can instantiate CKB with RandGP-UCB exploration to obtain a new algorithm called \texttt{CKB-Rand} with performance guarantees given by Theorem~\ref{thm:rand_exp}. Note that
RandGP-UCB with other distributions $\hat{\mathcal{D}}$  can also satisfy the sufficient condition (as discussed in Appendix~\ref{app:rand}).

\section{Further Improvement on Constraint Violations}


In the previous section, we have shown that our proposed CKB algorithm (Algorithm~\ref{alg:CKB}) is able to attain a sublinear regret (nearly the same order as the unconstrained case) and a sublinear constraint violation when employed with various exploration strategies. One natural question to ask is whether we can further improve the constraint violation bound. In the following, we will show that with a minor modification in Algorithm~\ref{alg:CKB}, one can achieve a bounded and even zero constraint violation by trading the regret slightly (but still the same order as before). The modification is to introduce a slackness given by $\epsilon$ in the dual update in Algorithm~\ref{alg:CKB}, i.e., $\phi_{t+1} = \text{Proj}_{[0,\rho]}\left[\phi_t +  \frac{1}{V}(\bar{g}_{t}(x_t) +\epsilon ) \right]$ with $\epsilon \le \delta/2$. Intuitively speaking, this can be viewed as if one is working on a new pessimistic constraint function. After obtaining the constraint violation under this new hypothetic constraint, one can subtract $\epsilon T$ to find the true constraint violation under the true function $g$. The catch here is that one needs also change the baseline problem to the following one: $\max_{\pi} \{ \mathbb{E}_{\pi}[f(x)]:\mathbb{E}_{\pi}\left[g(x) \right]+ \epsilon  \le 0\}$ so that it matches the new pessimistic constraint function. Let $\pi_{\epsilon}^*$ be the optimal solution to this problem and the obtained regret is only with respect to $\pi_{\epsilon}^*$ rather than $\pi^*$. Thus, we need to further bound the following difference $T \mathbb{E}_{\pi^*}\left[f(x)\right] - T \mathbb{E}_{\pi_{\epsilon}^*}\left[f(x)\right]$ to obtain the true regret bound. To this end, we have the following result.
\begin{claim}
\label{clm:T1}
$T \mathbb{E}_{\pi^*}\left[f(x)\right] - T \mathbb{E}_{\pi_{\epsilon}^*}\left[f(x)\right] \le  \frac{2BT\epsilon}{\delta}$. 
\end{claim}
To show this, we let $\pi_{\epsilon}(x):= (1-\frac{\epsilon}{\delta})\pi^*(x) + \frac{\epsilon}{\delta} \pi_0(x)$, where $\pi^*$ is the optimal solution to the original baseline problem and $\pi_0$ is the Slater's policy satisfying Slater's condition. First, we note that $\pi_{\epsilon}$ is a feasible solution to the new baseline problem introduced above. To see this, we note that $\pi_{\epsilon}(x) \ge 0$ and 
\begin{align*}
     \mathbb{E}_{\pi_{\epsilon}}\left[g(x)\right] = (1-\frac{\epsilon}{\delta})\mathbb{E}_{\pi^*}\left[g(x)\right] + \frac{\epsilon}{\delta} \mathbb{E}_{\pi_0}\left[g(x)\right] \le 0 + (-\epsilon) = -\epsilon.
\end{align*}
Since $\pi_{\epsilon}^*$ is the optimal solution while $\pi_{\epsilon}$ is a feasible one, we have 
\begin{align*}
    T \mathbb{E}_{\pi^*}\left[f(x)\right] - T \mathbb{E}_{\pi_{\epsilon}^*}\left[f(x)\right] &\le T \mathbb{E}_{\pi^*}\left[f(x)\right] - T \mathbb{E}_{\pi_{\epsilon}}\left[f(x)\right]\\
    &= T\left(\mathbb{E}_{\pi^*}\left[f(x)\right] - (1-\frac{\epsilon}{\delta})\mathbb{E}_{\pi^*}\left[f(x)\right] - \frac{\epsilon}{\delta} \mathbb{E}_{\pi_0}\left[f(x)\right] \right)\\
    &\le \frac{2BT\epsilon}{\delta},
\end{align*}
where in the last step, we use the boundedness of $f$. Therefore, one can properly choose $\epsilon$ such that the subtraction of $\epsilon T$ in the constraint violation can cancel the leading term $O(\sqrt{T})$ (hence bounded or even zero constraint violation) while only incurring an additional additive term of the same order in the regret.

\section{Discussion on Alternative Method}
\label{sec:dis}
To the best of our knowledge, there exist two popular methods for analyzing constrained bandits or MDPs. They are both based on primal-dual optimization and only differ in the analysis techniques. The first one is based on convex optimization tools as in~\citep{efroni2020exploration,ding2021provably} and our paper. The other one is based on Lyapunov-drift arguments as in~\citep{liu2021efficient,wei2021provably,liu2021learning}. For simplicity, we call the first method \emph{convex-opt} method and the second one as \emph{Lyapunov-drift} method. Before we provide further discussion, one thing to note is that all existing works only deal with UCB-type exploration for tabular or linear functions, while our paper is the first one that studies general functions with general exploration strategies. 

Now we first briefly explain the main idea behind the Lyapunov-drift method when applied to our setting (for the UCB exploration only).
It basically has the same algorithm as the convex-opt method. One minor change is that in Lyapunov-drift method, the dual variable is not truncated by $\rho$ and  is denoted by $Q(t)$, since this dual update is similar to a typical queue length update in queueing theory, i.e., truncated at zero; see Algorithm~\ref{alg:Lya}. To bound the regret, Lyapunov-drift method decomposes it as the following one, where $\epsilon \le \delta/2$ is the slackness as in the last section.

\begin{algorithm}[t]
\caption{Algorithm in the Lyapunov-drift method}
\label{alg:Lya}
\DontPrintSemicolon
\KwIn{$V$, $\epsilon$, $Q(1) = 0$}
\For{$t\!=\!1,2,\dots, T$}{
Generate estimate $f_t(x), g_t(x)$ and truncate them to $\bar{f}_t, \bar{g}_t$ \;
{Pseudo-acquisition function:} $z_{t}(x) = \bar{f}_{t}(x) - \frac{1}{V}Q(t) \bar{g}_t(x)$\;
Choose action $x_t = \argmax_{x\in  \mathcal{X}} z_{t}(x)$; observe reward $r_t$, and cost $c_t$\;
{Update virtual queue:} $Q(t+1) = \left[Q(t) +  \bar{g}_{t}(x_t) + \epsilon  \right]_+ $ \;
Posterior model update: using observations to update model\;
}
\end{algorithm}
\begin{align}
\label{eq:reg_drift}
    \mathcal{R}_+(T) &= T \mathbb{E}_{\pi^*}\left[f(x)\right] -  \sum_{t=1}^{T} {f(x_t)}\nonumber\\
    &=\underbrace{T \mathbb{E}_{\pi^*}\left[f(x)\right] - T \mathbb{E}_{\pi_{\epsilon}^*}\left[f(x)\right] }_{\text{Term 1}} + \underbrace{{\sum_{t=1}^{T} \int_{x \in \mathcal{X}} \left( f(x) - \bar{f}_t(x)\right) \pi_{\epsilon}^*(x)
    \,dx }}_{\text{Term 2}}\nonumber\\
    &+\underbrace{{\sum_{t=1}^{T} \int_{x \in \mathcal{X}} \bar{f}_t(x) \pi^*_{\epsilon}(x)\,dx } - \bar{f}_t(x_t)}_{\text{Term 3}} + \underbrace{{\sum_{t=1}^{T}\bar{f}_t(x_t) -f(x_t) }}_{\text{Term 4}}.
\end{align}
From this, one can see that $\text{Term 1}$, $\text{Term 2}$, and $\text{Term 4}$ can be easily bounded under UCB-type exploration. In particular, by optimism and well-concentration of $\bar{f}_t$, one has $\text{Term 2} \le 0$ and $\text{Term 4} = \widetilde{O}(\sqrt{T})$ (we ignore $\gamma_T$ term in this section for simplicity). Moreover, $\text{Term 1}$ enjoys the bound as in Claim~\ref{clm:T1}. Thus, the only challenge is to bound $\text{Term 3}$, which cannot be naturally bounded by greedy selection as in the standard way, since in the constrained case, the greedy selection is with respect to the combined function. To handle this, one needs the following result, which not only helps to bound $\text{Term 3}$, but also is the key in bounding the constraint violation. 
\begin{lemma}
\label{lem:delta_drift}
Let $\Delta(t) := L(Q(t+1)) - L(Q(t)) = \frac{1}{2}(Q(t+1))^2 - \frac{1}{2}(Q(t))^2$.  For any $\pi$, we have 
\begin{align}
\label{eq:T3}
    {\Delta(t)} &\le - V\left({\int_{x\in \mathcal{X}} \bar{f}_t(x)\pi(x) \,dx } - \bar{f}_t(x_t)\right)+  \frac{1}{2}(G+\epsilon)^2 + Q(t){\left( \int_{x\in \mathcal{X}} \bar{g}_t(x) \pi(x) \,dx  + \epsilon\right)}. 
\end{align}
\end{lemma}
\begin{proof}
See Appendix~\ref{app:proof_delta}.
\end{proof}
Thus, one can see that the first term on the RHS of~\eqref{eq:T3} exists in $\text{Term 3}$ if one chooses $\pi = \pi_{\epsilon}^*$. By the optimism $\bar{g}_t(x) \le g(x)$ and the definition of $\pi_{\epsilon}^*$, with a telescope summation, one can easily bound $\text{Term 3}$, hence the regret bound. 

\noindent\textbf{Comparison in regret analysis.} Compared to the regret decomposition in our paper, i.e.,~\eqref{eq:t1} and~\eqref{eq:t2},~\eqref{eq:reg_drift} in the Lyapunov-drift method is more tailored to UCB-type exploration in the sense that the $\text{Term 3}$ is upper bounded separately using the optimism. As a result, it is unclear to us how to generalize it to handle general exploration strategies where one often need to bound $\text{Term 2} + \text{Term 3} + \text{Term 4}$ together and optimism does not hold in general. In contrast, our decomposition~\eqref{eq:t1} and~\eqref{eq:t2} basically keep the same fashion as in the unconstrained case, which enables us to utilize this structure to handle general exploration strategies.

We now turn to the constraint violation. By the virtual queue length update in Algorithm~\ref{alg:Lya}, the key step behind the constraint violation bound is to bound $Q(T+1)$. To see this, by the virtual queue length update in Algorithm~\ref{alg:Lya}, we have 
\begin{align*}
  Q(T+1) \ge \sum_{t=1}^T \bar{g}_t(x_t) + T\epsilon = \sum_{t=1}^T \bar{g}_t(x_t) - g(x_t) + g(x_t) + T\epsilon,
\end{align*}
which implies that 
\begin{align*}
    \sum_{t=1}^T g(x_t) &\le Q(T+1) + \sum_{t=1}^T \left(g(x_t) - \bar{g}_t(x_t)\right) -T\epsilon = Q(T+1) + \widetilde{O}(\sqrt{T})- T\epsilon,
\end{align*}
where in the last step we uses the well-concentration of $\bar{g}_t$.
To bound the remaining term $Q(T+1)$, the Lyapunov-drift method resorts to a classic tool in queueing theory, i.e., Hajek lemma~\citep{hajek1982hitting}, to bound the virtual queue length at time $T+1$. The idea behind it is simple: if the queue length  drift $\Delta(t)$ is negative whenever the queue length is large, then $Q(T+1)$ is bounded. To establish the negative drift, one resorts to~\eqref{eq:T3} again by choosing $\pi = \pi_0$. By the definition of $\pi_0$ (Slater's policy), the optimism $\bar{g}_t(x) \le g(x)$ and boundedness of $\bar{f}_t$, one can easily establish a negative drift, and hence the constraint violation.

\noindent\textbf{Comparison in constraint violation analysis.} Instead of using Hajek lemma, we directly utilize the convex optimization tool to obtain the constraint violation as in~\citep{efroni2020exploration,ding2021provably}, which is conceptually simpler. Moreover, the current constraint violation analysis in the Lyapunov-drift method also relies on the optimism of $\bar{g}_t$, which does not hold in general explorations beyond UCB.  Finally, when applying Hajek lemma to bound the virtual queue length, there exists a subtlety that makes the standard expected version of Hajek lemma fail due to the correlation of virtual queue length $Q(t)$ and $\bar{g}_t$. We give more details on this subtlety in Appendix~\ref{app:subtlety}.

\section{Simulation Results}
\label{sec:simulation}


In this section, we conduct simulations to compare the performance of our algorithms (i.e., \texttt{CKB-UCB}, \texttt{CKB-TS}, and \texttt{CKB-Rand}, that is, CKB with GP-UCB, GP-TS, and RandGP-UCB explorations, respectively) with existing safe KB algorithms based on both synthetic and real-world datasets. In particular, we consider the two most recent safe KB algorithms: \texttt{StageOpt} [28] (which has a superior performance compared to \texttt{SafeOpt} [27]) and \texttt{SGP-UCB} [3]. Our goal is to 
show that our proposed CKB algorithms can trade a slight performance in constraint violation for improvement in the reward regret with a reduced computation complexity and  flexible implementations, i.e., a better complexity-regret-constraint trade-off.

\subsection{Synthetic Data and Light-Tailed Real-World Data}

\textbf{Synthetic Data.} The domain $\mathcal{X}$ is generated by discretizing $[0,1]$ uniformly into $100$ points. The objective function $f(\cdot) = \sum_{i=1}^p a_i k(\cdot,x_i) $ is generated by
uniformly sampling $a_i \in [-1,1]$ and support points $x_i \in \mathcal{X}$ with  $p = 100$. With the same manner, we generate the constraint function $g$.  
The kernel is $k_{se}$ with parameter $l = 0.2$. Other parameters include $B$, $R$ and $\gamma_t$ are set similar as in the unconstrained case (e.g.,~\citep{chowdhury2017kernelized}). 

\noindent\textbf{Light-Tailed Real-World Data.} We use the light sensor data collected in the CMU Intelligent Workplace in Nov 2005, which is available online as Matlab structure\footnote{\url{http://www.cs.cmu.edu/~guestrin/Class/10708-F08/projects/}} 
and contains locations of 41 sensors, 601 train samples and 192 test samples. We use it in the
context of finding the maximum average reading of the sensors.  In particular, $f$ is set as empirical average of the test samples, with $B$ set as its maximum, and $k$ is set as the empirical covariance of the normalized train samples. The constraint is given by $g(\cdot) = -f(\cdot) + h$ with $h = B/2$.


We perform $50$ trials (each with $T = 10,000$) and plot the mean of the cumulative regret along with the  error bars, as shown in Fig.~\ref{fig:regret}.

\begin{figure*}[t]\centering
		\begin{subfigure}{0.32\textwidth}
			\includegraphics[width = 1.8in]{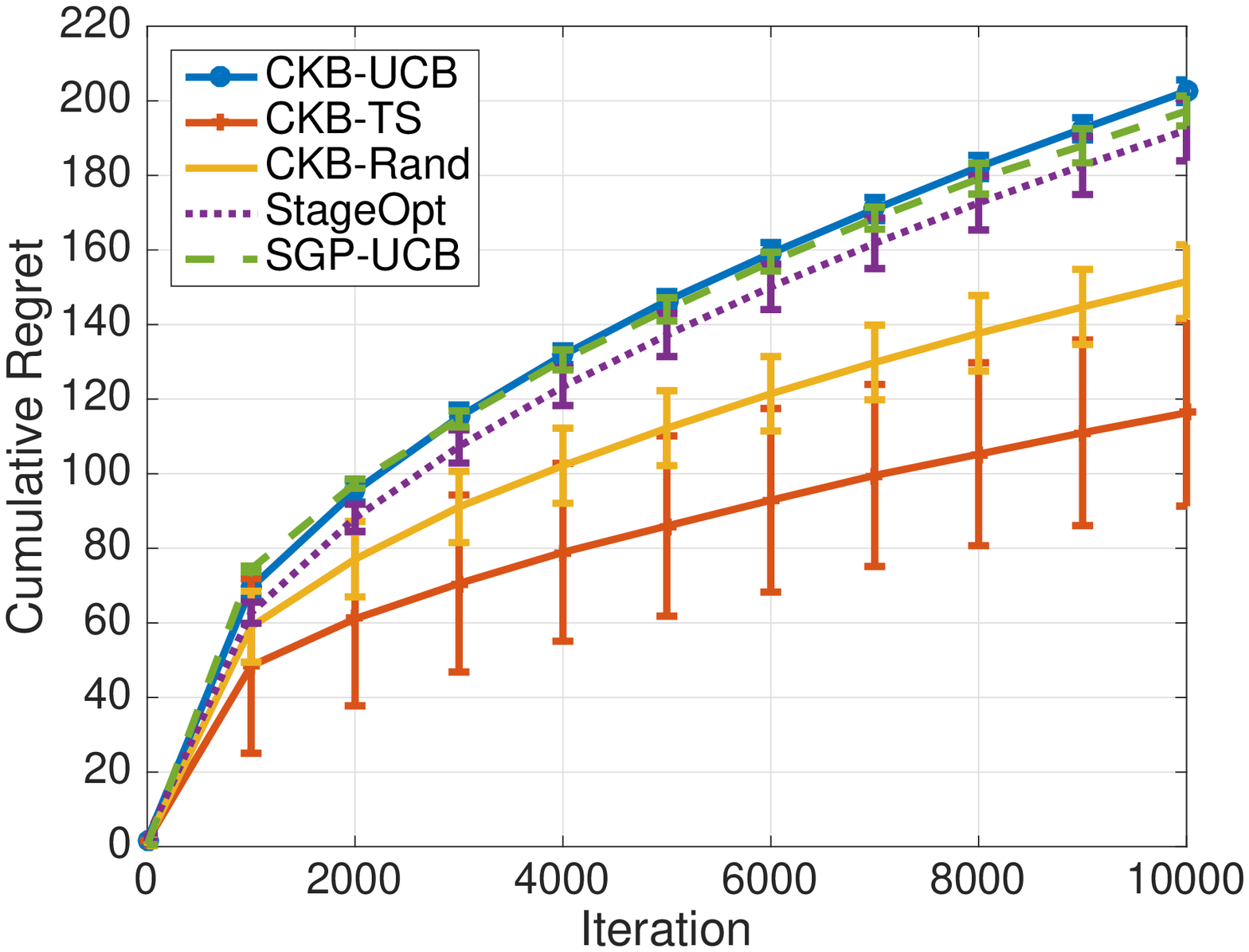}
			\caption{Regret on synthetic data }
		\end{subfigure}
		\begin{subfigure}{0.32\textwidth}
			\includegraphics[width = 1.8in]{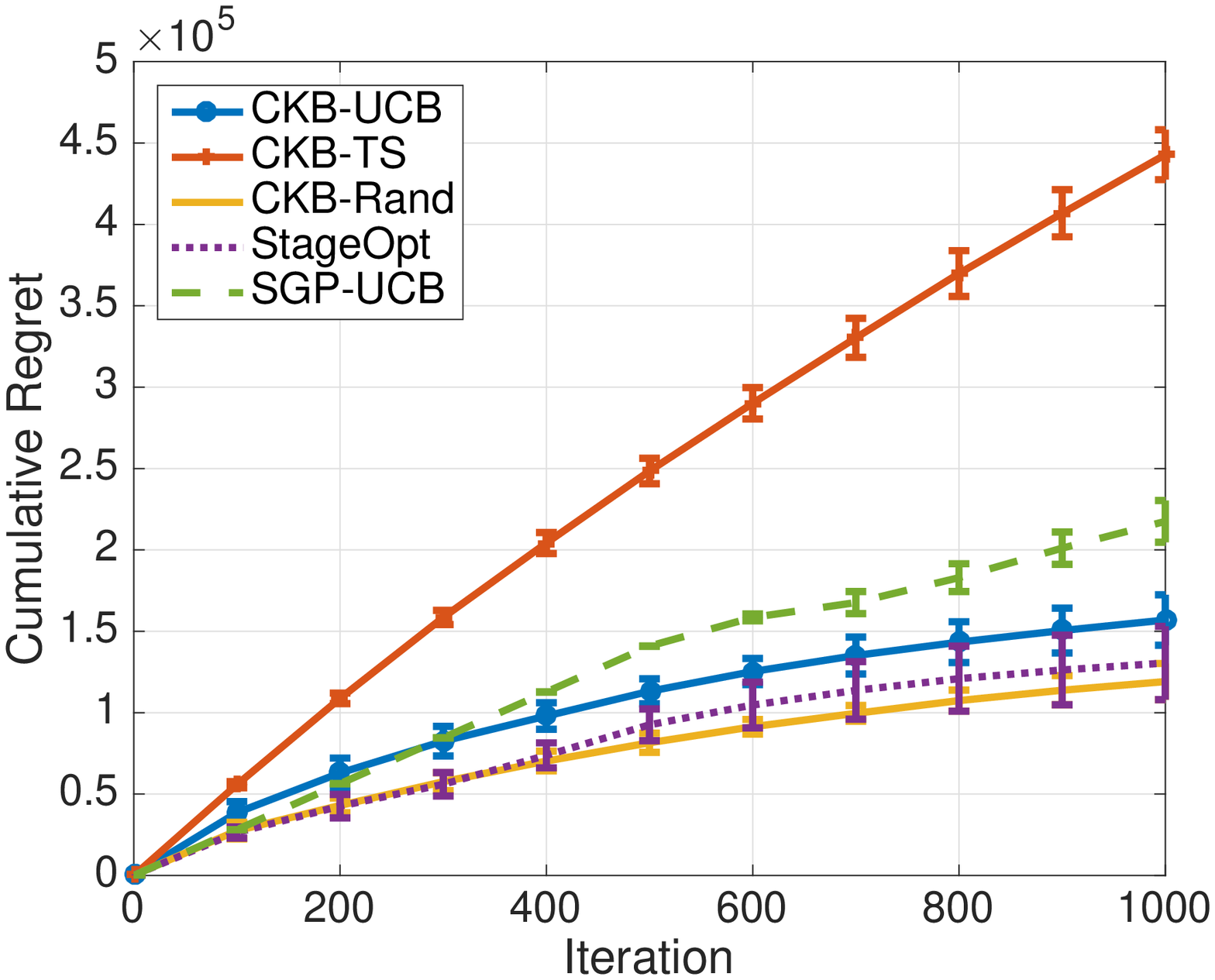}
			\caption{Regret on real-world data}
		\end{subfigure}
		\begin{subfigure}{0.32\textwidth}
			\includegraphics[width = 1.8in]{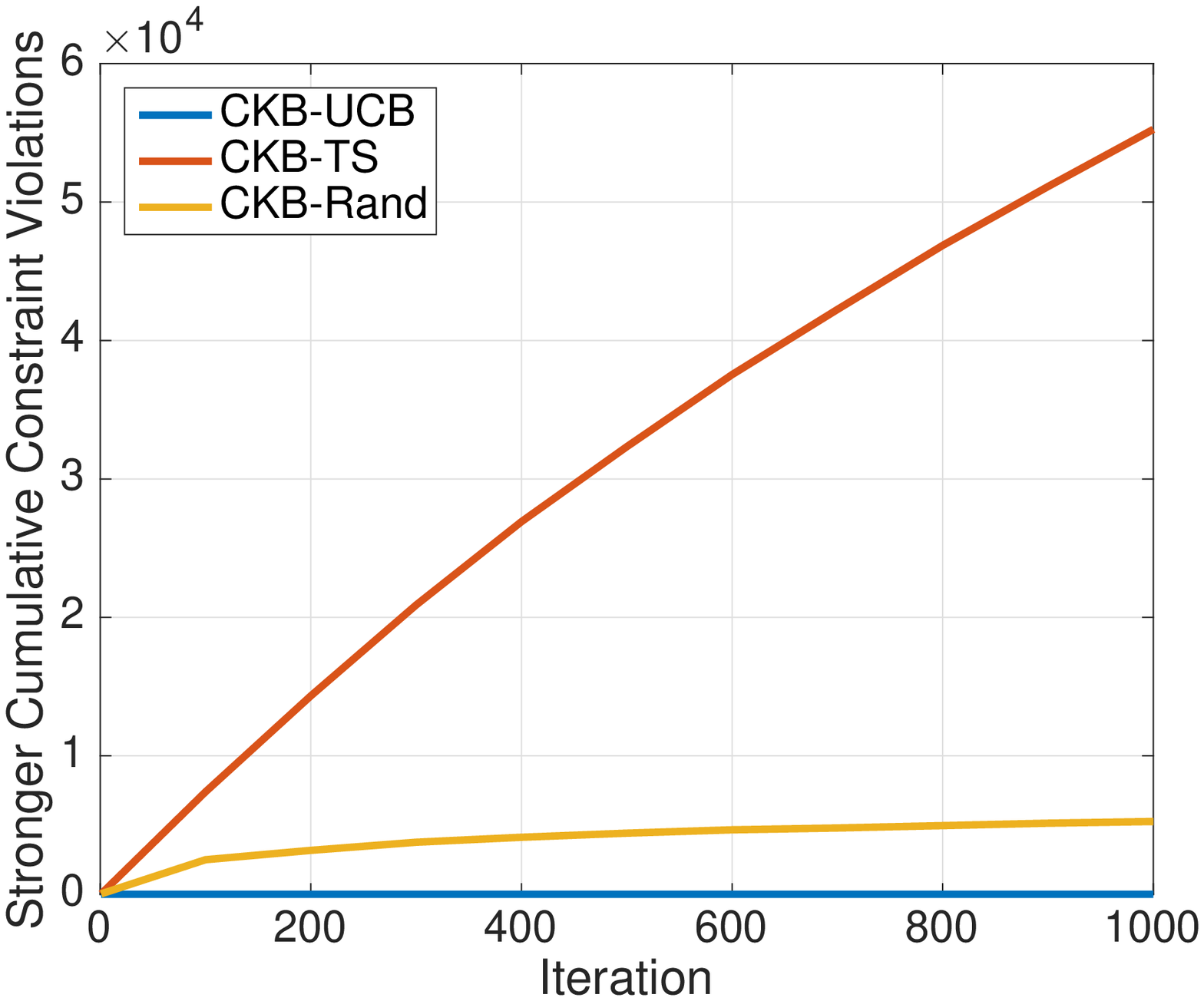}
			\caption{Constraint on real-world data}
		\end{subfigure}
		\caption{Experimental results on constrained kernelized bandits.}\label{fig:regret}
\end{figure*}

\noindent\textbf{Regret.} Our three CKB algorithms achieve a better (or similar) regret performance compared to the existing safe BO algorithms (see Figures~\ref{fig:regret}(a) and \ref{fig:regret}(b)). Among the three CKB algorithms, \texttt{CKB-Rand} appears to have reasonably good performance at all times.

\noindent\textbf{Constraint violation.} Since we have $\mathcal{V}(T) = 0$ under all the algorithms, we study the total number of rounds where the constraint is violated, denoted by $N$. In the synthetic data setting, our proposed CKB algorithms have $N \le 5$ over $T=10,000$ rounds; in the real-world data setting, \texttt{CKB-UCB} enjoys $N = 0$ and \texttt{CKB-Rand} has an average $N = 38$ over a horizon $T=1,000$. Furthermore, we plot the stronger cumulative constraint violations given by $\sum_{t=1}^T [g(x_t)]_+$ as shown in Figure~\ref{fig:regret}(c), from which we can see that all CKB algorithms achieve sublinear performance even with respect to this stronger metric.

\noindent\textbf{Practical considerations.} Our proposed CKB algorithms have the same computational complexity as the unconstrained case. In particular, they scale linearly with the number of actions in the discrete-domain case\footnote{For a continuous domain, as in the unconstrained case, one can resort to heuristic solvers (e.g., a combination of random sampling (cheap) and the "L-BFGS-B'" optimization method.). In fact, to attain the same order of regret bound, the solution to the acquisition maximization problem need not be exact. Instead, it only needs to maximize the acquisition function within $C/\sqrt{t}$ accuracy for some constant $C$ at each step. This will translate into an additional $C\sqrt{T}$ term in the regret bound.}. On the other hand, \texttt{StageOpt} scales quadratically due to the construction of the safe set, and \texttt{SGP-UCB} requires the additional random initialization stage, which leads to linear regret at the beginning of the learning process. Moreover, standard methods for improving the scalability of unconstrained KB can be naturally applied to our CKB algorithms. Finally, both \texttt{StageOpt} and \texttt{SGP-UCB} require the knowledge of a safe action (i.e., one that satisfies the constraint) in advance, and moreover, \texttt{StageOpt} requires $f$ to be Lipschitz and needs to estimate the Lipschitz constant, which impacts the robustness.  In contrast, CKB algorithms only require a mild Slater's condition as in Assumption~\ref{ass:slater_bandit}, which does not necessarily require the existence of a safe action. 


\subsection{Heavy-Tailed Real-World Data}
We further compare different constrained KB algorithms in a new real-world dataset, which demonstrates a heavy-tailed noise. Note that sub-Gaussian noise is required in all the existing theoretical works (including our work). We use this dataset to test the robustness of various constrained KB algorithms. The experimental results tend to show that our three CKB algorithms are more robust in terms of heavy-tailed noise, which is common in practical applications. 
The detail of this real-world dataset is deferred to Appendix~\ref{app:finance}.


\noindent\textbf{Regret.} We plot both cumulative regret and time-average regret in this setting (see Figures~\ref{fig:finace} (a) and (b)). We can observe that in the presence of heavy-tailed noise, our three CKB algorithms have significant performance gain over existing safe KB algorithms. 

\noindent\textbf{Constraint violation.} We focus on the strong metric, i.e., the number of rounds where the constraint is violated, denoted by $N$. We have that \texttt{CKB-Rand} enjoys an average $N = 21$ and \texttt{CKB-UCB} has an average $N = 47$ within the horizon of $T=10,000$. We also plot stronger cumulative constraint violations  given by $\sum_{t=1}^T [g(x_t)]_+$ as shown in Figure~\ref{fig:finace}(c), from which we can see that all CKB algorithms achieve sublinear performance even with respect to this stronger metric. 


\begin{figure*}[t]\centering
		\begin{subfigure}[b]{0.3\textwidth}
			\includegraphics[width = 1.8in]{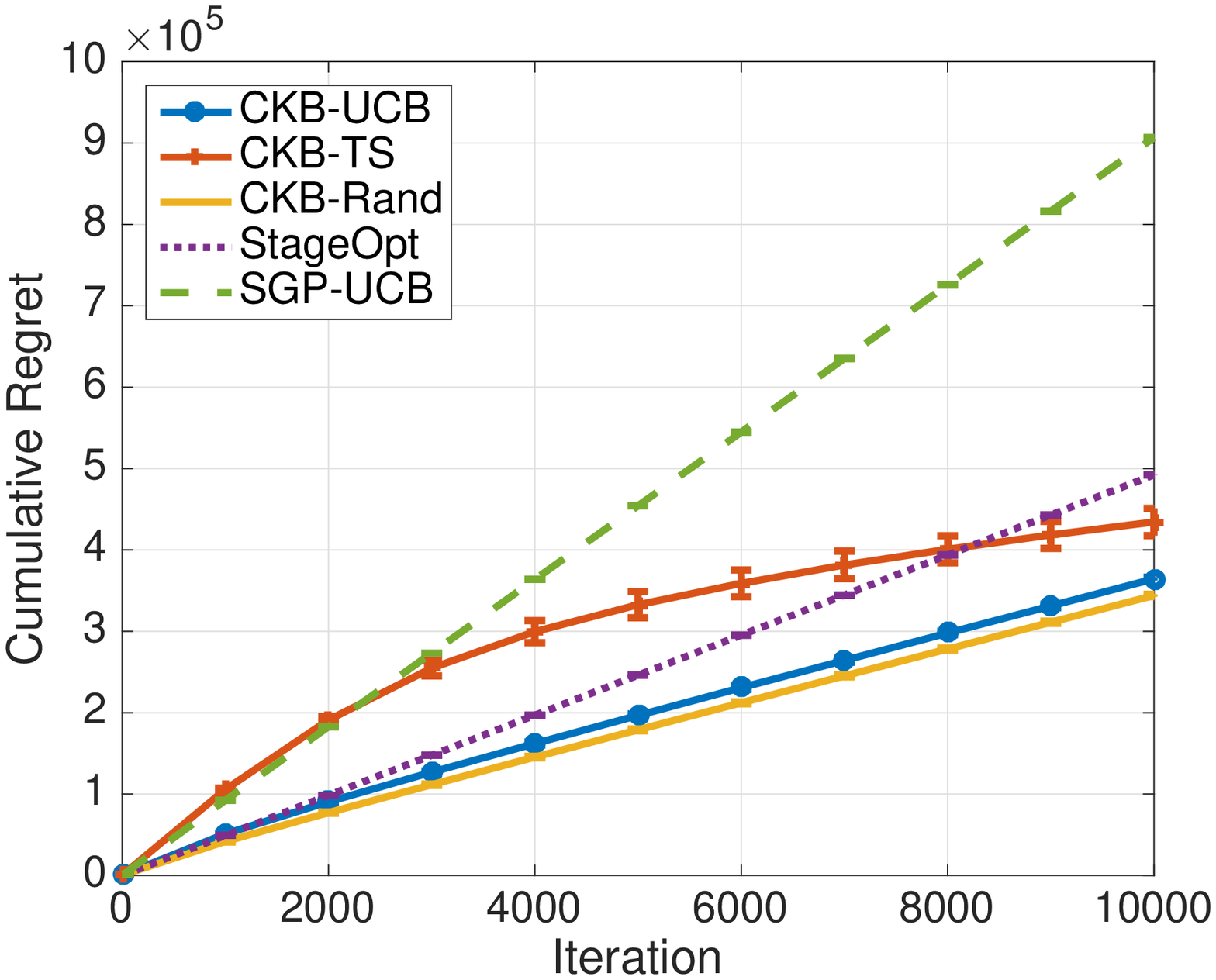}
			\caption{Regret on finance data }
		\end{subfigure}
		\quad
		\begin{subfigure}[b]{0.3\textwidth}
			\includegraphics[width = 1.8in]{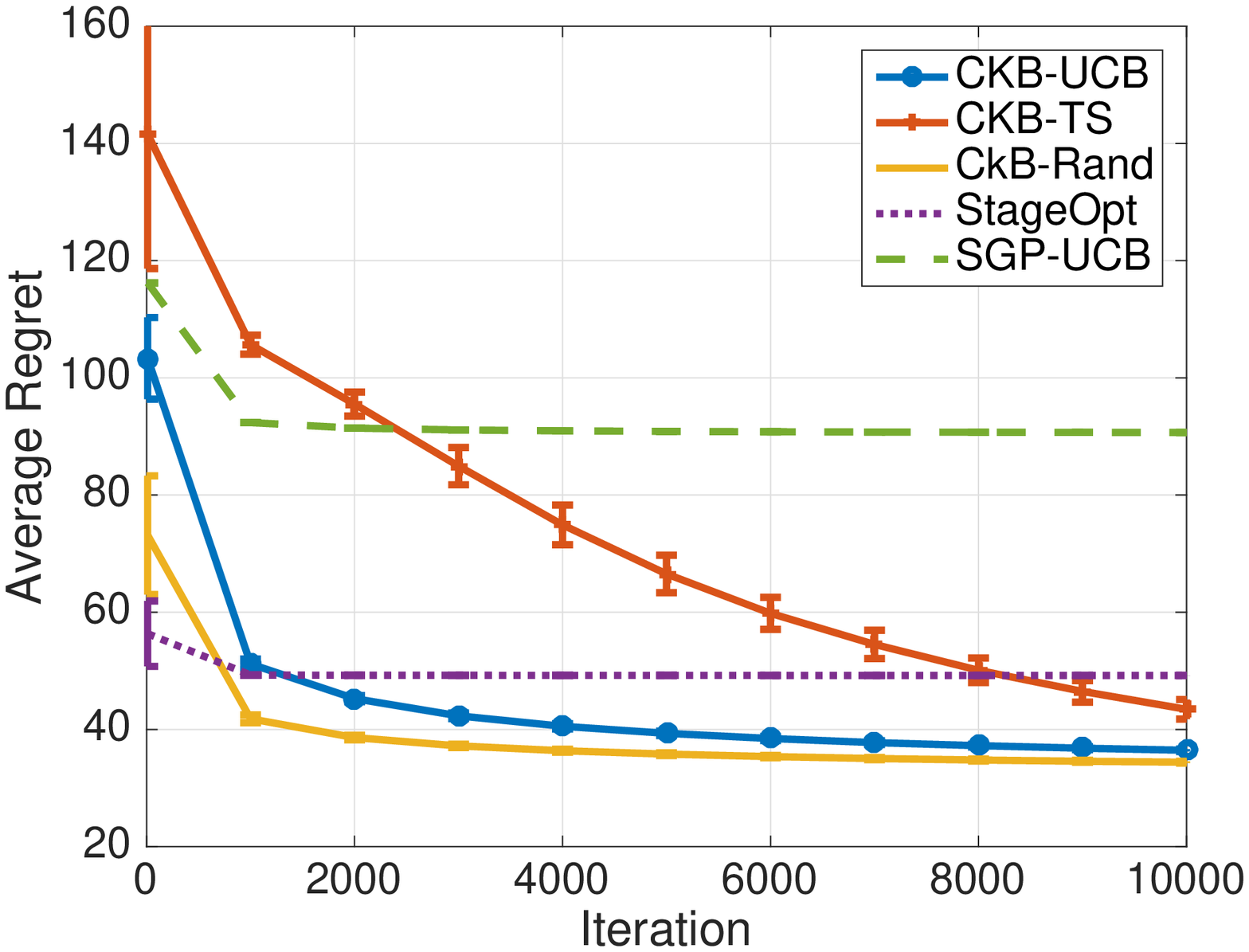}
			\caption{Avg. regret on finance data}
		\end{subfigure}
		\quad
		\begin{subfigure}[b]{0.3\textwidth}
			\includegraphics[width = 1.8in]{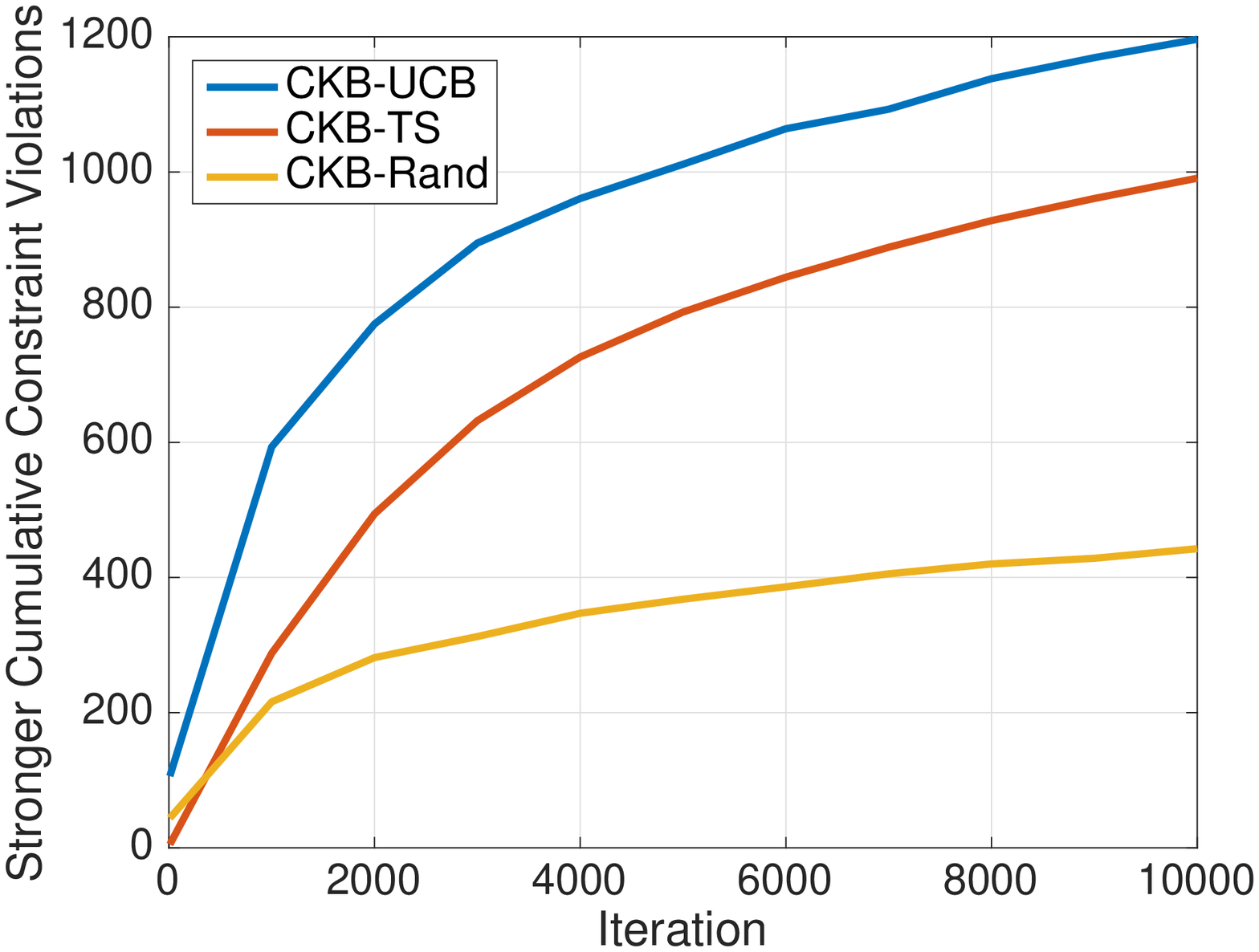}
			\caption{Constraint on finance data}
		\end{subfigure}
		 \vspace{-2mm}
		\caption{Experimental results on constrained kernelized bandits under heavy-tailed finance data.}\label{fig:finace}
		\vspace{-2mm}
\end{figure*}


\section{Proof of Theorem~\ref{thm:bandit_UCB}}

Before we present the proof, we first obtain the following lemma on the dual variable. 
\begin{lemma}
\label{lem:dual}
Under the update rule of $\phi_t$ in Algorithm~\ref{alg:CKB}, we have for any $\phi \in [0,\rho]$,
\begin{align*}
    \sum_{t=1}^{T}\bar{g}_t(x_t) (\phi -\phi_t) \le \frac{V}{2}(\phi_{1} - \phi)^2  + \sum_{t=1}^{T}\frac{1}{2V}\bar{g}_t(x_t)^2.
\end{align*}
\end{lemma}
\begin{proof}
By the dual variable update rule in Algorithm~\ref{alg:CKB} and the non-expansiveness of projection to $[0,\rho]$, we have 
\begin{align*}
    (\phi_{t+1} - \phi)^2 &\le (\phi_t + \frac{1}{V}\bar{g}_t(x_t) - \phi)^2\\
    &=(\phi_t - \phi)^2 + \frac{2}{V}\bar{g}_t(x_t) (\phi_t - \phi) + \frac{1}{V^2} \bar{g}_t(x_t)^2.
\end{align*}

Summing over $T$ steps and multiplying both sides by $\frac{V}{2}$, we have 
\begin{align*}
  \frac{V}{2}(\phi_{T+1} - \phi)^2 - \frac{V}{2}(\phi_{1} - \phi)^2  \le \sum_{t=1}^{T}\bar{g}_t(x_t) (\phi_t - \phi) + \sum_{t=1}^{T}\frac{1}{2V}\bar{g}_t(x_t)^2.
\end{align*}

Hence, 
\begin{align}
    \sum_{t=1}^{T}\bar{g}_t(x_t) (\phi - \phi_t) \le \frac{V}{2}(\phi_{1} - \phi)^2  + \sum_{t=1}^{T}\frac{1}{2V}\bar{g}_t(x_t)^2,
\end{align}
which completes the proof.
\end{proof}

Now, we are ready to present the proof of Theorem~\ref{thm:bandit_UCB}.

\begin{proof}[Proof of Theorem~\ref{thm:bandit_UCB}]

Under Slater condition in Assumption~\ref{ass:slater_bandit}, we have the boundedness of the optimal dual solution by standard convex optimization analysis (cf.~\citep[Theorem 8.42]{beck2017first})
\begin{align*}
    0 \le \phi^* \le \frac{\left(\mathbb{E}_{\pi^*}\left[f(x)\right] - \mathbb{E}_{\pi_0}\left[f(x)\right]\right)}{\delta} \le \frac{2B}{\delta},
\end{align*}
where the last inequality holds by the boundedness of $f(x)$. Note that the reason why we can use convex analysis is that $\mathbb{E}_{\pi}\left[h(x)\right]$ for any fixed $h$ is a linear function with respect to $\pi$ (and is thus convex). Now, we  turn to establish a bound over $\mathcal{R}(T) + \phi\sum_{t=1}^{T}g(x_t)$. First, note that 
\begin{align}
     &\mathcal{R}(T) + \phi \sum_{t=1}^{T} g(x_t)\nonumber\\
    =& T \mathbb{E}_{\pi^*}\left[f(x)\right] - \sum_{t=1}^{T} f(x_t) + \phi\sum_{t=1}^{T}g(x_t)\nonumber\\
    =& T \mathbb{E}_{\pi^*}\left[f(x)\right] - \sum_{t=1}^{T}\bar{f}_t(x_t) + \sum_{t=1}^{T}\bar{f}_t(x_t) - f(x_t)+\phi\sum_{t=1}^{\tau}g(x_t)\label{eq:rt_plus_start}.
\end{align}
We can further bound~\eqref{eq:rt_plus_start} by using Lemma~\ref{lem:dual}. In particular, we have 
\begin{align}
  & T \mathbb{E}_{\pi^*}\left[f(x)\right] - \sum_{t=1}^{T}\bar{f}_t(x_t) + \sum_{t=1}^{T}\bar{f}_t(x_t) - f(x_t)+\phi\sum_{t=1}^{\tau}g(x_t)\nonumber\\
    \lep{a}& \sum_{t=1}^{T}\mathbb{E}_{\pi^*}\left[f(x)\right] - \phi_t \mathbb{E}_{\pi^*}\left[g(x)\right]  - \sum_{t=1}^{T}\bar{f}_t(x_t) + \sum_{t=1}^{T}\bar{f}_t(x_t) - f(x_t)+\phi\sum_{t=1}^{T}g(x_t)\nonumber\\
     \ep{b}& \sum_{t=1}^{T}\mathbb{E}_{\pi^*}\left[f(x)\right] - \phi_t \mathbb{E}_{\pi^*}\left[g(x)\right]  - \left(\sum_{t=1}^{T}\bar{f}_t(x_t) - \phi_t \bar{g}_t(x_t)\right) + \sum_{t=1}^{T}\bar{f}_t(x_t) - f(x_t)\nonumber \\
     &+ \phi \left(\sum_{t=1}^{T}g(x_t)\nonumber - \bar{g}_t(x_t)\right) +  \sum_{t=1}^{T}\bar{g}_t(x_t) (\phi - \phi_t)\nonumber\\
     \lep{c} & \sum_{t=1}^{T}\mathbb{E}_{\pi^*}\left[f(x)\right] - \phi_t \mathbb{E}_{\pi^*}\left[g(x)\right]  - \left(\sum_{t=1}^{T}\bar{f}_t(x_t) - \phi_t \bar{g}_t(x_t)\right) + \sum_{t=1}^{\tau}\bar{f}_t(x_t) - f(x_t)\nonumber\\
     &+\phi \left(\sum_{t=1}^{T}g(x_t)\nonumber - \bar{g}_t(x_t)\right) + \frac{V}{2}(\phi_{1} - \phi)^2  + \sum_{t=1}^{T}\frac{1}{2V}\bar{g}_t(x_t)^2\nonumber\\
     \ep{d}& \mathcal{T}_1 + \mathcal{T}_2 + \frac{V}{2}\phi^2 + \frac{1}{2V}TG^2,\label{eq:rt_plus_end}
\end{align}
where (a) holds since $\phi_t \ge 0$ and $\mathbb{E}_{\pi^*}\left[g(x)\right] \le 0$; (b) holds by adding and subtracting terms; (c) follow from Lemma~\ref{lem:dual} to bound the last term; (d) holds by the fact $\phi_1 = 0$, the boundedness of $\bar{g}_t$ and the definitions of $\mathcal{T}_1$ and $\mathcal{T}_2$, i.e.,
\begin{align}
    &\mathcal{T}_1 = \sum_{t=1}^{T} (\mathbb{E}_{\pi^*}\left[f(x)\right]- \phi_t \mathbb{E}_{\pi^*}\left[g(x)\right])  - \sum_{t=1}^{T} (\bar{f}_t(x_t) - \phi_t \bar{g}_t(x_t))\label{eq:t1_def},\\
    &\mathcal{T}_2 = \sum_{t=1}^{T} (\bar{f}_t(x_t) - f(x_t))
     +\phi \sum_{t=1}^{T} (g(x_t) - \bar{g}_t(x_t))\label{eq:t2_def}.
\end{align}

Plugging~\eqref{eq:rt_plus_end} into~\eqref{eq:rt_plus_start}, yields for any $\phi \in [0,\rho]$,
\begin{align}
\label{eq:rt_plus}
    \mathcal{R}(T) + \phi \sum_{t=1}^{T} g(x_t) \le \mathcal{T}_1 + \mathcal{T}_2 + \frac{V}{2}\phi^2 + \frac{1}{2V}TG^2.
\end{align}
First, assume that we already have a bound on $\mathcal{T}_1 + \mathcal{T}_2$ , i.e., $\mathcal{T}_1 + \mathcal{T}_2 \le \chi(T,\phi)$ with high probability, and $\chi(T,\phi)$ is an increasing function in $\phi$. This directly leads to the following inequality (with $V = G\sqrt{T}/\rho$) for any $\phi \in [0,\rho]$: 
\begin{align}
\label{eq:decomp_proof}
    \mathcal{R}_+(T) + \phi\sum_{t=1}^T g(x_t) \le \chi(T,\phi) + \frac{\phi^2 G\sqrt{T}}{2\rho} + \frac{\rho G\sqrt{T}}{2}.
\end{align}
Based on this key inequality, we can analyze both regret and constraint violation. 

\textbf{Regret}. We can simply choose $\phi = 0$ in~\eqref{eq:decomp_proof}, and obtain that with high probability 
\begin{align}
\label{eq:regret_final}
    \mathcal{R}_+(T) = O\left(\chi(T,0) + \rho G \sqrt{T}\right).
\end{align}

\textbf{Constraint violation}. To obtain the bound on $\mathcal{V}(T)$, inspired by~\citep{efroni2020exploration}, we will resort to tools from constrained convex optimization. First, we have $\frac{1}{T}\sum_{t=1}^T f(x_t) =   \mathbb{E}_{\pi'}\left[f(x)\right]$ and $\frac{1}{T}\sum_{t=1}^T g(x_t) = \mathbb{E}_{\pi'}\left[g(x)\right]$ for some probability measure $\pi'$ by the convexity of probability measure. As a result, we have 
\begin{align}
\label{eq:cons_condition}
   \mathbb{E}_{\pi^*}\left[f(x)\right] - \mathbb{E}_{\pi'}\left[f(x)\right] + \rho \left[\mathbb{E}_{\pi'}\left[g(x)\right]\right]_+  = \frac{1}{T} \mathcal{R}_+(T) + \frac{1}{T}\phi\sum_{t=1}^T g(x_t) \le \frac{\chi(T,\rho) + \rho G\sqrt{T}}{T}, 
\end{align}
where $\left[a\right]_+ : =\max\{0,a\}$, and the first equality holds by choosing $\phi = \rho$ if $\sum_{t=1}^Tg(x_t) \ge 0$, and otherwise $\phi=0$, and the second inequality holds by upper bounding RHS of~\eqref{eq:decomp} with $\phi = \rho$ since~\eqref{eq:decomp_proof} holds for all $\phi \in [0,\rho]$ and $\chi(T,\phi)$ is increasing in $\phi$. 

Then, we will apply the following useful lemma, which is adapted from Theorem 3.60 in~\citep{beck2017first}.
\begin{lemma}
\label{lem:beck}
Consider the following convex constrained problem $h(\pi^*) = \max_{\pi \in \mathcal{C}}\{h(\pi): w(\pi) \le 0\}, $where both $h$ and $w$ are convex over the convex set $\mathcal{C}$ in a vector space. Suppose $h(\pi^*)$ is finite and there exists a slater point $\pi_0$ such that $w(\pi_0) \le -\delta$, and a constant $\rho \ge 2 \kappa^*$, where $\kappa^*$ is the optimal dual variable, i.e., $\kappa^* = \argmin_{\lambda \ge 0}( \max_{\pi} h(\pi) - \kappa w(\pi) )$. Assume that $\pi' \in \mathcal{C}$ satisfies 
\begin{align}
\label{eq:cons_lem}
    h(\pi^*) - h(\pi') + \rho \left[w(\pi')\right]_+ \le \varepsilon,
\end{align}
for some $\varepsilon > 0$, then we have $[w(\pi')]_+ \le 2\varepsilon/\rho$.
\end{lemma}

Thus, since~\eqref{eq:cons_condition} satisfies~\eqref{eq:cons_lem} and $\mathbb{E}_{\pi}\left[h(x)\right]$ for any fixed $h$ is a linear function with respect to $\pi$, by Lemma~\ref{lem:beck}, we have 
\begin{align}
\label{eq:cons_vio_fina}
\mathcal{V}(T) = O\left(\frac{1}{\rho}\chi(T,\rho) + G\sqrt{T}\right).
\end{align}

We are only left to bound $\mathcal{T}_1 + \mathcal{T}_2$ by $\chi(T,\phi)$. To this end, we will resort to standard concentration results for GP bandits. First, by~\citep[Theorem 2]{chowdhury2017kernelized}, we have the following lemma. 
\begin{lemma}
\label{lem:gp_concen}
Fix $\alpha \in (0,1]$, with probability at least $1-\alpha$, the followings hold simultaneously for all $t \in [T]$ and all $x \in \mathcal{X}$
\begin{align*}
    &|f(x) - \mu_{t-1}(x)| \le \beta_t \sigma_{t-1}(x), \quad |g(x) - \widetilde{\mu}_{t-1}(x)| \le \widetilde{\beta}_t \widetilde{\sigma}_{t-1}(x),
\end{align*}
\end{lemma}

Thus, based on this lemma and the definition of GP-UCB exploration, we have with high probability, $f_t(x) \ge f(x)$ and $g_t(x) \le g(x)$ for all $t \in [T]$ and $x \in \mathcal{X}$. This directly implies that $\bar{f}_t(x) \ge f(x)$ and $\bar{g}_t(x) \le g(x)$ for all $t \in [T]$ and $x \in \mathcal{X}$ (i.e., optimistic estimates), which holds by $|f(x)| \le B$ and $|g(x)|\le G$ and the way of truncation in Algorithm~\ref{alg:CKB}. Now, to bound $\mathcal{T}_1$ in~\eqref{eq:t1_def}, we have
\begin{align*}
    \mathcal{T}_1 &= \sum_{t=1}^{T} (\mathbb{E}_{\pi^*}\left[f(x)\right]- \mathbb{E}_{\pi^*}\left[\bar{f}_t(x)\right] + \mathbb{E}_{\pi^*}\left[\bar{f}_t(x)\right] - \bar{f}_t(x_t) ) \\
    &+ \phi_t \sum_{t=1}^{T} (\bar{g}_t(x_t)- \mathbb{E}_{\pi^*}\left[\bar{g}_t(x)\right] + \mathbb{E}_{\pi^*}\left[\bar{g}_t(x)\right] -  \mathbb{E}_{\pi^*}\left[{g}(x)\right] )\\
    &\lep{a} \sum_{t=1}^{T} \left(\mathbb{E}_{\pi^*}\left[\bar{f}_t(x)\right] - \bar{f}_t(x_t)\right) + \phi_t \sum_{t=1}^{T} \left(\bar{g}_t(x_t)- \mathbb{E}_{\pi^*}\left[\bar{g}_t(x)\right]\right)\\
    &= \sum_{t=1}^T \left( \mathbb{E}_{\pi^*}\left[\bar{f}_t(x)\right] - \phi_t \mathbb{E}_{\pi^*}\left[\bar{g}_t(x)\right] - (\bar{f}_t(x_t) -\phi_t \bar{g}_t(x_t)) \right)\\
    &\lep{b} 0,
\end{align*}
where (a) holds by the fact that estimates are optimistic, i.e., $\bar{f}_t(x) \ge f(x)$ and $\bar{g}_t(x) \le g(x)$ for all $t \in [T]$ and $x \in \mathcal{X}$; (b) holds by the greedy selection of Algorithm~\ref{alg:CKB}.

Now, we turn to bound $\mathcal{T}_2$. In particular, we have 
\begin{align}
    \mathcal{T}_2 &\lep{a} \sum_{t=1}^T 2\beta_t\sigma_{t-1}(x_t) + \phi \sum_{t=1}^T2\widetilde{\beta}_t \widetilde{\sigma}_{t-1}(x_t)\nonumber\\
    &\lep{b} O\left(\beta_T \sqrt{T\gamma_T} + \phi \widetilde{\beta}_T \sqrt{T\widetilde{\gamma}_T}\right),\label{eq:sum_variance}
\end{align}
where (a) holds by Lemma~\ref{lem:gp_concen} and the definition of GP-UCB exploration, i.e., $f_t(x) = \mu_{t-1}(x) + \beta_t \sigma_{t-1}(x)$ and $g_t(x) = \widetilde{\mu}_{t-1}(x) - \widetilde{\beta}_t \widetilde{\sigma}_{t-1}(x)$. Note that truncation also does not affect this step; (b) holds by Cauchy-Schwartz inequality and the bound of sum of predictive variance (cf.~\citep[Lemma 4]{chowdhury2017kernelized}). Note that we have also used the fact that $\beta_t$ and $\widetilde{\beta}_t$ is increasing in $t$. 

Putting the bounds on $\mathcal{T}_1$ and $\mathcal{T}_2$ together, we have obtained that with high probability 
\begin{align*}
    \mathcal{T}_1 + \mathcal{T}_2 \le \chi(T,\phi) := O\left(\beta_T \sqrt{T\gamma_T} + \phi \widetilde{\beta}_T \sqrt{T\widetilde{\gamma}_T}\right).
\end{align*}
Finally, plugging $\chi(T,0)$ into~\eqref{eq:regret_final}, yield the regret bound as follows (note that $\beta_t  = B + R\sqrt{2(\gamma_{t-1} + 1 + \ln(2/\alpha))}$)
\begin{align*}
    \mathcal{R}_+(T) = O\left(B\sqrt{T\gamma_T} + \sqrt{T\gamma_T(\gamma_T + \ln (2/\alpha))} + \rho G \sqrt{T}\right),
\end{align*}
and plugging $\chi(T,\rho)$ into~\eqref{eq:cons_vio_fina}, yields the bound on constraint violation as 
\begin{align*}
    \mathcal{V}(T) = O\left(\left(1+\frac{1}{\rho}\right)\left( C\sqrt{T\hat{\gamma}_T} + \sqrt{T\hat{\gamma}_T (\hat{\gamma}_T + \ln(2/\alpha))}\right) + G\sqrt{T} \right),
\end{align*}
where $C := \max\{B,G\}$ and $\hat{\gamma}_T := \max\{\gamma_T,\widetilde{\gamma}_T\}$. Hence, it completes the proof.
\end{proof}

\section{Proof of Theorem~\ref{thm:rand_exp}}
Before we present the proof, we introduce a new notation to make the presentation easier. In particular, we let $h(\pi):= \mathbb{E}_{\pi}\left[h(x)\right]$ for any function $h$ and $\pi_t$ is a dirac delta function at the point $x_t$.
\begin{proof}[Proof of Theorem~\ref{thm:rand_exp}]
As shown in the proof of Theorem~\ref{thm:bandit_UCB}, all we need to do is to find a high probability bound over $\mathcal{T}_1 + \mathcal{T}_2$ under the sufficient condition in Assumption~\ref{ass:sufficient}. Under our newly introduced notation, we have 
\begin{align}
\label{eq:t1t2_proof}
    \mathcal{T}_1 + \mathcal{T}_2 =  &\sum_{t=1}^T\left(z_{\phi_t}(\pi^*) - \hat{z}_{\phi_t}(\pi_t) + \hat{z}_{\phi}(\pi_t)- z_{\phi}(\pi_t)\right) := \sum_{t=1}^T d_t,
\end{align}
where $z_{\phi_t}(\cdot):=f(\cdot) - \phi_t g(\cdot)$ and $\hat{z}_{\phi_t}(\cdot):=\bar{f}_t(\cdot) - \phi_t \bar{g}_t(\cdot)$, and similar definitions for $z_{\phi}$ and $\hat{z}_{\phi}$. 

Let $\Delta_{\phi_t}(\pi):= z_{\phi_t}(\pi^*) - z_{\phi_t}(\pi) = (f(\pi^*) - \phi_t g(\pi^*)) - (f(\pi)-\phi_t g(\pi))$. Then,  we define the `undersampled' set as 
\begin{align*}
    \bar{S}_{t}:=\{\pi \in \Pi: \alpha_{\phi_t}(\pi):=c_{f,t} \sigma_{t-1}(\pi) + \phi_tc_{g,t}\widetilde{\sigma}_{t-1}(\pi) \ge \Delta_{\phi_t}(\pi)\},
\end{align*}
 where $c_{f,t} = (c_{f,t}^{(1)} + c_{f,t}^{(2)})$ and $c_{g,t} = (c_{g,t}^{(1)} + c_{g,t}^{(2)})$ (similarly $\alpha_{\phi}(\pi):=c_{f,t} \sigma_{t-1}(\pi) + \phi c_{g,t}\widetilde{\sigma}_{t-1}(\pi)$). Let $u_t = \argmin_{\pi\in \bar{S}_t} \alpha_{\phi_t}(\pi)$. Thus, conditioned on $E^{est}$ and $E_t^{conc}$,  we have 
 \begin{align}
    d_t &= z_{\phi_t}(\pi^*) - \hat{z}_{\phi_t}(\pi_t) + \hat{z}_{\phi}(\pi_t)- z_{\phi}(\pi_t)\nonumber\\
    &=z_{\phi_t}(\pi^*) - z_{\phi_t}(u_t) + z_{\phi_t}(u_t) - \hat{z}_{\phi_t}(\pi_t) + \hat{z}_{\phi}(\pi_t)- z_{\phi}(\pi_t) \nonumber\\
    &= \Delta_{\phi_t}(u_t) + z_{\phi_t}(u_t) - \hat{z}_{\phi_t}(\pi_t) + \hat{z}_{\phi}(\pi_t)- z_{\phi}(\pi_t)\nonumber\\
    &\lep{a} \Delta_{\phi_t}(u_t) +
    \hat{z}_{\phi_t}(u_t) - \hat{z}_{\phi_t}(\pi_t) +
    \alpha_{\phi_t}(u_t)  + \alpha_{\phi}(\pi_t)\nonumber\\
    &\lep{b} \Delta_{\phi_t}(u_t) +  \alpha_{\phi_t}(u_t) + \alpha_{\phi}(\pi_t)\nonumber\\
    &\lep{c} 2\alpha_{\phi_t}(u_t) + \alpha_{\phi}(\pi_t),\label{eq:dt}
\end{align}
where (a) holds since under event $E^{est} \cap E_t^{conc}$, for all $x$, $|f(x) - f_t(x)| \le (c_{f,t}^{(1)} + c_{f,t}^{(2)})\sigma_{t-1}(x)$ and $|g(x) - g_t(x)| \le (c_{g,t}^{(1)} + c_{g,t}^{(2)})\widetilde{\sigma}_{t-1}(x)$ and the facts that $|g(x) - \bar{g}_t(x)| \le |g(x) - {g}_t(x)|$ since $|g(x)| \le G$ and $|{f}(x)-\bar{f}_t(x)| \le |f(x) - {f}_t(x)|$ since $|f(x)| \le B$; (b) holds by the greedy selection in Algorithm~\ref{alg:CKB}; (c) follows from $u_t \in \bar{S}_t$. Thus, conditioned on $E^{est}$, we have
\begin{align*}
    \ext{d_t} &= \ext{d_t I\{E_t^{conc}\}} + \ext{d_t I\{\bar{E}_t^{conc}\}}\\
    &\lep{a}  \ext{r_t I\{E_t^{conc}\}} + (4B+4\rho G)p_{2,t}\\
    &\lep{b} \ext{\alpha_{\phi}(\pi_t)} + 2\alpha_{\phi_t}(u_t) + (4B+4\rho G)p_{2,t}\\
    &\lep{c} \ext{\alpha_{\phi}(\pi_t)} + 2\frac{\ext{\alpha_{\phi_t}(\pi_t)}}{\pt{x_t \in \bar{S}_t}} + (4B+4\rho G)p_{2,t}\\
    &\ep{d} \left(1+\frac{2}{\pt{\pi_t \in \bar{S}_t}}\right)\ext{\alpha_{\rho}(\pi_t)} + (4B+4\rho G)p_{2,t},
\end{align*}
where (a) holds by definition of $p_{2,t}$, the fact that $\phi, \phi_t \le \rho$ and the boundedness of functions; (b) follows from Eq.~\eqref{eq:dt} and the fact that given $\mathcal{F}_{t-1}$, $\alpha_{\phi_t}(u_t)$ is deterministic; (c) holds by the following argument: $\ext{\alpha_{\phi_t}(\pi_t)} \ge \ext{\alpha_{\phi_t}(\pi_t) | \pi_t \in \bar{S}_t} \pt{\pi_t \in \bar{S}_t} \ge  {\alpha_{\phi_t}(u_t)} \pt{\pi_t \in \bar{S}_t} $, which holds by the definition of $u_t$ and the fact that $\alpha_{\phi_t}(u_t)$ and $S_t$ are both $\mathcal{F}_{t-1}$-measurable; (d) holds by definition $\alpha_{\rho}(\pi_t):= c_{f,t} \sigma_{t-1}(\pi_t) + \rho c_{g,t}\widetilde{\sigma}_{t-1}(\pi_t)$ and the fact that both $\phi, \phi_t$ are bounded by $\rho$. Hence, the key is to find a lower bound on the probability $\pt{\pi_t\in \bar{S}_t}$. In particular, conditioned on $E^{est}$, we have 
\begin{align*}
    &\pt{\pi_t \in \bar{S}_t} \\
    &\gep{a} \pt{\hat{z}_{\phi_t}(\pi^*) \ge \max_{\pi_j \in {S}_t}\hat{z}_{\phi_t}(\pi_j), E_t^{conc} }\\
    &\gep{b} \pt{\hat{z}_{\phi_t}(\pi^*) \ge z_{\phi_t}(\pi^*), E_t^{conc}  }\\
    &\ge \pt{\hat{z}_{\phi_t}(\pi^*) \ge z_{\phi_t}(\pi^*)} - \pt{\bar{E}_t^{conc}}\\
    &\ge \pt{\bar{f}_t(\pi^*) \ge f(\pi^*), \bar{g}_t(\pi^*) \le g(\pi^*) } - \pt{\bar{E}_t^{conc}}\\
    &\ep{c} \pt{f_t(\pi^*) \ge f(\pi^*), {g}_t(\pi^*) \le g(\pi^*) } - \pt{\bar{E}_t^{conc}}\\
    &\gep{d} \pt{f_t(\pi^*) \ge \mu_{t-1}(\pi^*) + c_{f,t}^{(1)}\sigma_{t-1}(\pi^*), g_t(\pi^*) \le \widetilde{\mu}_{t-1}(\pi^*) - c_{g,t}^{(1)}\widetilde{\sigma}_{t-1}(\pi^*)  } - \pt{\bar{E}_t^{conc}}\\
    &=\pt{E_{t}^{anti}} - \pt{\bar{E}_t^{conc}}\\
    &=  p_3-p_{2,t},
\end{align*}
where (a) holds by the greedy selection in Algorithm~\ref{alg:CKB} and $\pi^* \in \bar{S}_t$ since $\Delta_{\phi_t}(\pi^*) = 0$. Note that $S_t$ is the complement of the `undersampled' set $\bar{S}_t$;  (b) holds given $E^{est} \cap E_t^{conc}$, for all $\pi_j \in S_t$ $\hat{z}_{\phi_t}(\pi_j) \le z_{\phi_t}(\pi_j) + \alpha_{\phi_t}(\pi_j) \le z_{\phi_t}(\pi_j) + \Delta_{\phi_t}(\pi_j) = z_{\phi_t}(\pi^*)$; (c) holds since $|g(x)| \le G$ for all $x$ and $|f(x)| \le B$ for all $x$; (d) holds since under $E^{est}$, we have $f(x) \le \mu_{t-1}(x^*) + c_{1,f}\sigma_{t-1}(x^*)$ and $g(x) \ge \widetilde{\mu}_{t-1}(x^*) - c_{1,g}\widetilde{\sigma}_{t-1}(x^*) $ for all $x$.

Putting everything together, we have now arrived at that conditioned on $E^{est}$, 
\begin{align}
    \ext{d_t} &\le \ext{\alpha_{\rho}(x_t)}\left(1+ \frac{2}{p_3-p_{2,t}}\right) + (4B+4\rho G)p_{2,t}\nonumber\\
    &\le  \frac{1}{p_4}\ext{\alpha_{\rho}(x_t)} + (4B+4\rho G)p_{2,t}\label{eq:bar_dt}.
\end{align}
where the last inequality follows from the boundedness condition in the sufficient condition. In order to obtain a high probability bound, inspired by~\citep{chowdhury2017kernelized}, we will resort to martingale techniques. Let us define the following terms
\begin{definition}
Define $Y_0 = 0$, and for all $t=1,\ldots, T$,
\begin{align*}
    &\bar{d}_t = d_t \mathcal{I}\{E^{est}\}\\
    &X_t = \bar{d}_t - \frac{1}{p_4} \alpha_{\rho}(x_t) - (4B+4\rho G)p_{2,t}\\
    &Y_t = \sum_{s=1}^t X_s,
\end{align*}
where $\mathcal{I}\{\cdot\}$ is the indicator function. 
\end{definition}

Now, we can show that $\{Y_t\}_t$ is a super-martingale with respect to filtration $\mathcal{F}_{t}$. To this end, we need to show that for any $t$ and any possible $\mathcal{F}_{t-1}$, $\ex{Y_t - Y_{t-1}|\mathcal{F}_{t-1}} \le 0$, i.e., $\ext{\bar{d}_t} \le \frac{1}{p_4} \ext{\alpha_{\rho}(x_t)} + (4B+4\rho G)p_{2,t}$. For $\mathcal{F}_{t-1}$ such that $E^{est}$ holds, we already obtained the required inequality as in Eq.~\eqref{eq:bar_dt}. For $\mathcal{F}_{t-1}$ such that $E^{est}$ does not hold, the required inequality trivially holds since the LHS is zero. Now, we turn to show that $\{Y_t\}_t$ is a bounded incremental sequence, i.e., $|Y_t - Y_{t-1}| \le M_t$ for some constant $M_t$.  We first note that 
\begin{align*}
    |Y_t - Y_{t-1}| &= |X_t| \le |\bar{d}_t| + \frac{1}{p_t}\alpha_{\rho}(x_t) + (4B+4\rho G)p_{2,t}\\
    &= |\bar{d}_t| + \frac{1}{p_4}\left(c_{f,t}\sigma_{t-1}(x_t) + \rho c_{g,t} \widetilde{\sigma}_{t-1}(x_t)\right) +  (4B+4\rho G)p_{2,t}\\
    &\lep{a} (4B+4\rho G) + \frac{1}{p_4} (c_{f,t} + \rho c_{g,t}) + (4B+4\rho G)p_{2,t} \\
    &\le \frac{1}{p_4}(c_{f,t}+\rho c_{g,t})(4B+4\rho G):= M_t,
\end{align*}
where (a) holds since $\bar{d}_t \le d_t \le (4B+4\rho G)$, $\sigma_{t-1}(x_t) \le \sigma_0(x_t) \le 1$ and $\widetilde{\sigma}_{t-1}(x_t) \le \widetilde{\sigma}_0(x_t) \le 1$. Thus, we can apply Azuma-Hoeffding inequality to obtain that with probability at least $1-\alpha$,
\begin{align*}
    \sum_{t=1}^T \bar{r}_t &\le \sum_{t=1}^T \frac{1}{p_4}\alpha_{\rho}(x_t) + \sum_{t=1}^T (4B+4\rho G)p_{2,t} + \sqrt{2\ln(1/\delta) \sum_{t=1}^T M_t^2 }\\
    &\lep{a} \frac{1}{p_4}\sum_{t=1}^T \alpha_{\rho}(x_t) + C'(4B+4\rho G) + \frac{(c_f(T)+\rho c_g(T))(4B+4\rho G)}{p_4}\sqrt{2T\ln(1/\delta)},
\end{align*}
where (a) we have used the boundedness condition.
Note that since $E^{est}$ holds with probability at least $1- p_1$ for all $t$ and $x$. By a union bound, we have with probability at least $1-\alpha - p_1$,
\begin{align}
    \sum_{t=1}^T {d}_t &\le \frac{1}{p_4}\sum_{t=1}^T \alpha_{\rho}(x_t) + C'(4B+4\rho G) + \frac{(c_{f}(T)+\rho c_{g}(T))(4B+4\rho G)}{p_4}\sqrt{2T\ln(1/\delta)}\nonumber\\
    &=O\left(\frac{1}{p_4}\sum_{t=1}^T (c_f(T)\sigma_{t-1}(x_t) + \rho c_g(T)\widetilde{\sigma}_{t-1}(x_t)) +  \frac{(c_f(T)+\rho c_g(T))\kappa }{p_4}\sqrt{2T\ln(1/\delta)}\right)\nonumber\\
    &=O\left(\frac{1}{p_4}c_f(T)\sqrt{T\gamma_T}  + \frac{1}{p_4}\rho c_g(T)\sqrt{T\widetilde{\gamma}_T} +  \frac{(c_f(T)+\rho c_g(T))\kappa }{p_4}\sqrt{2T\ln(1/\delta)}\right),\label{eq:dt_final}
\end{align}
where $\kappa:= 4B+4\rho G$. Plugging~\eqref{eq:dt_final} into~\eqref{eq:t1t2_proof}, we obtain that 
\begin{align*}
    \mathcal{T}_1 + \mathcal{T}_2 &\le O\left(\frac{1}{p_4}c_f(T)\sqrt{T\gamma_T}  + \frac{1}{p_4}\rho c_g(T)\sqrt{T\widetilde{\gamma}_T} +  \frac{(c_f(T)+\rho c_g(T))\kappa }{p_4}\sqrt{2T\ln(1/\delta)}\right)\\
    &:= \chi(T,\phi).
\end{align*}
Note that here $\chi(T,\phi)$ is independent of $\phi$ since we have bounded it by $\rho$ in the analysis. Finally, plugging $\chi(T,\phi)$ into~\eqref{eq:regret_final} and~\eqref{eq:cons_vio_fina} yields the results of Theorem~\ref{thm:rand_exp}.
\end{proof}

\section{Conclusion}
We studied kernelized bandits with unknown soft constraints to attain a finer complexity-regret-constraint trade-off. To this end, we presented a general framework for constrained KB with soft constraints via primal-dual optimization. Armed with our developed sufficient condition, this framework not only allows us to design provably efficient (i.e., sublinear reward regret and sublinear total constraint violation) CKB algorithms with both UCB and TS explorations, but presents a unified method to design new effective ones. By introducing slackness, our algorithm can also attain a bounded or even zero constraint violation while still achieving a sublinear regret.
We further perform simulations on both synthetic data and real-world data that corroborate our theoretical results. 
Along the way, we also present the first detailed discussion on two existing methods for analyzing constrained bandits and MDPs by highlighting interesting insights.
One interesting future work is to generalize our results to kernelized MDPs~\citep{yang2020function}. 

\bibliographystyle{unsrtnat}
\bibliography{main,isit,aaai}

\newpage

\begin{appendix}

\section{Flexible Implementations of RandGP-UCB}
\label{app:rand}
In this section, we will give more insights on the choices of $\hat{\mathcal{D}}$, i.e., sampling distribution for $\hat{Z}_t$. In particular, we consider the unconstrained case for useful insights with black-box function being $f$. By the definition of RandGP-UCB, for each $t$, the estimate under RandGP-UCB is given by 
\begin{align*}
   f_t(x) = {\mu}_{t-1(x)} + {Z}_t \sigma_{t-1}(x),
\end{align*}
where $Z_t \sim \mathcal{D}$. 
First, by Lemma~\ref{lem:gp_concen}, we have with high probability 
\begin{align*}
    f(x) \le \mu_{t-1} + \beta_t \sigma_{t-1}(x),
\end{align*}
which directly implies that in order to guarantee $E_t^{anti}$ happens with a positive probability, one needs to make sure that $\mathbb{P}(Z_t \ge \beta_t) \ge p_3 > 0$. Thus, one simple choice of $\mathcal{D}$ is a  uniform discrete distribution between $[0,2\beta_t]$ with $N$ points. Then, it can be easily checked that $\pt{E_t^{anti}} \ge p_3 >0$ and also $\pt{E_t^{conc}} = 1$ with $c_{f,t}^{(2)} = 2\beta_t$. In addition to uniform discrete distribution, one can also use discrete Gaussian distribution within a range $[L,U]$ as long as $U$, $L$ are properly chosen. Of course, there are many other choices as long as the insight shown above is satisfied, and hence RandGP-UCB provides a lot of flexibility in the algorithm design.

\section{Details on Heavy-Tailed Real-World Data}\label{app:finance}

This dataset is the adjusted closing price of 29 stocks from January 4th, 2016 to April 10th 2019.  We use it in the context of identifying the most profitable stock in a given pool of stocks.  As verified in~\cite{chowdhury2019bayesian}, the rewards follows from heavy-tailed distribution.  We take the empirical mean of stock prices as our objective function $f$
and empirical covariance of the normalized stock prices as our kernel function $k$. The noise is estimated by taking the difference between the raw prices and its empirical mean (i.e., $f$), with $R$ set as the maximum. The constraint is given by $g(\cdot) = - f(\cdot) + h$ with $h = 100$ (i.e., $h \approx B/2$). 
We perform $50$ trials (each with $T= 10,000$) and plot the mean along with the error bars.

\section{Proof of Lemma~\ref{lem:delta_drift}}
\label{app:proof_delta}
\begin{proof}
Note that by the update rule of the virtual queue in Algorithm~\ref{alg:Lya} and non-expansiveness of projection, we have \begin{align*}
    \Delta(t) \le  Q(t)(\bar{g}_t(x_{t}) + \epsilon) + \frac{1}{2}\left(\bar{g}_t(x_t) + \epsilon\right)^2.
\end{align*}
Now we will bound the RHS as follows. 
\begin{align*}
     &Q(t)(\bar{g}_t(x_{t}) + \epsilon) + \frac{1}{2}\left(\bar{g}_t(x_t) + \epsilon\right)^2\\
     \lep{a}&Q(t)(\bar{g}_t(x_{t} ) + \epsilon) + \frac{1}{2}(G+\epsilon)^2\\
     =& -V \bar{f}_t(x_{t} ) + Q(t)\bar{g}_t(x_{t})  + Q(t)\epsilon + V\bar{f}_t(x_{t}) +  \frac{1}{2}(G+\epsilon_{t})^2\\
     \lep{b}& - V\int_{x\in\mathcal{X}} \bar{f}_t(x) \pi(x) \,dx +  Q(t) \int_{x\in \mathcal{X}} \bar{g}_t(x) \pi(x) \,dx+Q(t)\epsilon + V\bar{f}_t(x_{t}) +  \frac{1}{2}(G+\epsilon)^2,
\end{align*}
where (a) holds by the boundedness of $\bar{g}_t$; (b) holds by the greedy selection in Algorithm~\ref{alg:Lya}. Reorganizing the term, yields the required result.
\end{proof}

\section{Subtlety in Applying Hajek Lemma to Constraint Violation}
\label{app:subtlety}
As stated before, the key step behind the constraint violation is to establish a negative drift of the virtual queue and then by Hajek lemma, one can show that the virtual queue is bounded in expectation, which in turn can be used to establish a zero constraint violation with a proper choice of slackness variable (i.e., $\epsilon$) in the virtual queue update. However, the negative drift condition in the standard Hajek lemma (cf. Lemma 11 in~\cite{liu2021efficient}) requires a conditional expectation, i.e., condition on all large enough $Q$, the expected drift is negative. Then, if one directly applies the standard Hajek lemma, she would proceed as follows. The goal is to show that  $\ex{\Delta(t) \mid Q(t) = Q} \le -c Q$ for all large $Q$ and $c$ is some positive constant. Recall the bound on $\Delta(t)$ in~\eqref{eq:T3}, by the boundedness and let $\pi = \pi_0$, the key is to show that 
\begin{align}
    \ex{ \int_{x\in \mathcal{X}} \bar{g}_t(x) \pi_0(x) \,dx + \epsilon \mid Q(t) = Q} \le -c.
\end{align}
To illustrate the idea, we simply suppose that the Slater's condition is satisfied at a single point $x_0$ and $\epsilon = 0$. To show the above inequality, she may choose the following direction.
\begin{align*}
   \ex{\bar{g}_t(x_0) | Q(t) = Q} = \underbrace{\ex{\bar{g}_t(x_0) - g(x_0) | Q(t) = Q}}_{\text{Term (i)}} + \underbrace{\ex{g(x_0) | Q(t) = Q}}_{\text{Term (ii)}} \le -c.
\end{align*}


For $\text{Term (ii)}$, it is easily bounded by $\text{Term (ii)} \le -\delta$ via Slater's condition since $g(\cdot)$ is a fixed function. To bound $\text{Term (i)}$, she may resort to the standard self-normalized inequality for linear bandits and the definition of UCB exploration (cf.~\cite{abbasi2011improved}). By these standard results, she can show that for any fixed $\alpha \in (0,1]$, the following holds:
\begin{align}
\label{eq:prob_opt}
    \mathbb{P}\{\forall x, \forall t, \bar{g}_t(x) \le g(x) \} \ge 1-\alpha.
\end{align}
 That is, $\bar{g}_t$ is optimistic with respect to $g$. Then, by setting $\alpha = 1/T$ and using the boundedness assumption of both $\bar{g}_t$ and $g$, she may conclude that  $\text{Term (i)} = O(1/T)$.
 Unfortunately, the bound on $\text{Term (i)}$ is ungrounded since it is obtained by treating the conditional expectation in $\text{Term (i)}$ as an unconditional expectation. The subtlety here is that one cannot remove the condition on $Q(t)$ in $\text{Term (i)}$, since $\bar{g}_t$ is \emph{not} independent of $Q(t)$ as both of them depend on the randomness before time $t$. Given a particular $Q(t)$, it roughly means that we are taking expectation conditioned on a particular history (i.e., a sample-path). Under this particular history,~\eqref{eq:prob_opt} does not necessarily hold, and moreover, the concentration of $\bar{g}_t$ given $Q(t)$ is hard to compute in this case. As a result, the conditional expectation for $\text{Term (i)}$ is hard to compute in general. 
 
 
 
 \vspace{6pt}
 \textbf{One correct way.} Instead of applying the standard expected version of Hajek lemma, one can consider removing the expectation in Hajek lemma by directly showing that $\bar{g}_t(x_0) \le -c$ almost surely under the ``good event''. This is exactly the approach used in \citep{liu2021learning} (cf. Lemma 5.6). In this way, one can show that with a high probability (i.e., under good event), a negative drift exists and hence the constraint violation bound with high probability.
 
\end{appendix}

\end{document}